\documentclass[11pt]{article}

\usepackage[margin=1in]{geometry}
\usepackage[numbers,compress]{natbib}
\usepackage[utf8]{inputenc} 
\usepackage[T1]{fontenc}    
\usepackage{hyperref}       
\usepackage{url}            
\usepackage{booktabs}       
\usepackage{amsfonts}       
\usepackage{nicefrac}       
\usepackage{microtype}      
\usepackage[table]{xcolor}  
\usepackage{amsmath, amssymb, amsthm}
\usepackage{mathtools}
\usepackage{algorithm}
\usepackage{algorithmic}
\usepackage{enumitem}
\usepackage{graphicx}
\usepackage{subcaption}
\usepackage{tabularx}
\usepackage{array}
\usepackage{longtable}
\usepackage{multirow}
\usepackage{placeins}
\usepackage{titletoc}
\usepackage{etoc}

\theoremstyle{plain}
\newtheorem{theorem}{Theorem}[section]
\newtheorem{lemma}[theorem]{Lemma}
\newtheorem{proposition}[theorem]{Proposition}
\newtheorem{corollary}[theorem]{Corollary}

\theoremstyle{definition}
\newtheorem{definition}[theorem]{Definition}

\theoremstyle{remark}
\newtheorem{remark}[theorem]{Remark}

\newcommand{\calS}{\mathcal{S}}

\newcommand{\Pmeas}{\mathbb{P}}
\newcommand{\Qmeas}{\mathbb{Q}}
\newcommand{\E}{\mathbb{E}}
\newcommand{\R}{\mathbb{R}}
\newcommand{\pdata}{p_{\text{data}}}
\newcommand{\pref}{p_{\text{prior}}}

\newcommand{\Rthetamat}{R^{\theta}}
\newcommand{\Rhatmat}{\hat{R}}
\newcommand{\lambdafwd}{\lambda^{\text{fwd}}}
\newcommand{\lambdatheta}{\lambda^{\theta}}
\newcommand{\lambdahat}{\hat{\lambda}}
\newcommand{\dd}{\mathrm{d}}
\newcommand{\KLdiv}{\mathrm{KL}}
\newcommand{\Cat}{\text{Cat}}
\newcommand{\Poi}{\text{Poi}}
\newcommand{\indicator}{\mathbf{1}}
\newcommand{\given}{\mid}
\definecolor{oldred}{HTML}{F7C6C6}
\definecolor{newgreen}{HTML}{BFF0CA}
\definecolor{headergray}{HTML}{F1F1F1}
\newcommand{\oldword}[1]{\begingroup\setlength{\fboxsep}{1pt}\colorbox{oldred}{\strut #1}\endgroup}
\newcommand{\newword}[1]{\begingroup\setlength{\fboxsep}{1pt}\colorbox{newgreen}{\strut #1}\endgroup}
\newcolumntype{Y}{>{\raggedright\arraybackslash}X}
\let\standardsum\sum
\renewcommand{\sum}{\standardsum\limits}

\title{Neural Continuous-Time Markov Chain: Discrete Diffusion via Decoupled Jump Timing and Direction}

\author{%
  \textbf{Jingyuan Li}$^{2,3,4}$\thanks{Equal contribution.}
  \quad
  \textbf{Xiaoyi Jiang}$^{1,4}$\footnotemark[1]
  \quad
  \textbf{Fukang Wen}$^{1,4}$
  \quad
  \textbf{Wei Liu}$^{3}$ \\
  \textbf{Renqian Luo}
  \quad
  \textbf{Yi Zhu}$^{1,2,4}$
  \quad
  \textbf{Zuoqiang Shi}$^{1,2,4}$
  \quad
  \textbf{Pipi Hu}$^{2,4}$\thanks{Corresponding author.} \\[0.5ex]
  $^{1}$Tsinghua University
  \quad
  $^{2}$Beijing Institute of Mathematical Sciences and Applications \\
  $^{3}$Wuhan University
  \quad
  $^{4}$MathonAI
}

\date{}

\begin{document}

\maketitle

\begin{abstract}
Discrete diffusion models based on continuous-time Markov chains (CTMCs) have shown strong performance on language and discrete data generation, yet existing approaches typically parameterize the reverse rate matrix monolithically---through proxies such as concrete scores (SEDD) or clean-data predictions (MDLM, GIDD)---rather than aligning the parameterization with the intrinsic CTMC decomposition into jump timing and jump direction. We propose \textbf{Neural CTMC}, which exploits the underlying Poisson structure of CTMC dynamics by separately parameterizing the reverse process through an \emph{exit rate} (when to jump) and a \emph{jump distribution} (where to jump) via two dedicated network heads. We show that the evidence lower bound (ELBO) reduces to a path-space KL divergence between the true and learned reverse processes that factorizes into a Poisson KL for timing and a categorical KL for direction, and admits a tractable, gradient-equivalent and consistent loss. Experimentally, scored by Gemma2-9B, our pure-uniform Neural CTMC achieves $16.36$ generative perplexity on TinyStories (vs.\ GIDD $37.60$ and MDLM $42.66$). On OpenWebText, it attains the best perplexity at the same training-token budget across 16--128 sampling steps among the methods we compare (e.g., at 128 steps: Neural CTMC $183.6$ vs.\ MDLM $210.5$ and GIDD $249.8$). To facilitate reproducibility, we release our pretrained weights at \url{https://huggingface.co/Jiangxy1117/Neural-CTMC}.
\end{abstract}

\section{Introduction}
\label{sec:intro}

Diffusion models have achieved remarkable success in continuous domains such as images~\citep{sohl2015deep, ho2020denoising, song2021scorebased} and audio, establishing score-based generative modeling as a dominant paradigm. Extending this success to discrete data---language~\citep{li2022diffusionlm, gulrajani2024likelihood}, proteins~\citep{alamdari2023protein, avdeyev2023dirichlet}, and molecular graphs~\citep{vignac2023digress}---is an active and challenging research direction, since the continuous score function and Langevin dynamics that underpin continuous diffusion have no direct discrete counterpart.

Discrete diffusion models formulate the forward corruption and reverse generation as continuous-time Markov chains (CTMCs) on a finite state space $\calS = \{1,\ldots,S\}$. D3PM~\citep{austin2021structured} and Multinomial Diffusion~\citep{hoogeboom2021argmax} initiated this line by introducing discrete transition matrices in discrete time. \citet{campbell2022continuous} extended the framework to continuous time, deriving a variational bound from CTMC path measures. Subsequent work has explored different parameterizations of the reverse rate matrix $\Rthetamat_t(i,j)$: SEDD~\citep{lou2024discrete} learns a concrete score ratio $q_t(j)/q_t(i)$ analogous to continuous score matching~\citep{hyvarinen2005estimation, meng2022concrete}; MDLM~\citep{sahoo2024simple} simplifies training via clean-data prediction ($x_0$-parameterization) under a masked (absorbing) forward process; GIDD~\citep{ruette2025gidd} extends the $x_0$-parameterization to unify absorbing, uniform, and interpolating forward processes under a general framework; and concurrent efforts have explored discrete flow matching~\citep{campbell2024generative, gat2024discrete} and simplified masked diffusion~\citep{shi2024simplified, ou2025absorbdiscrete}. These methods have achieved increasingly strong results on language modeling benchmarks.

Despite this progress, a fundamental structural property of CTMCs has not been directly reflected in the reverse-process parameterization. Any rate matrix admits a unique factorization
\begin{equation*}
R_t(i,j) \;=\; \lambda_t(i) \cdot r_t(j \mid i), \qquad \lambda_t(i) = \textstyle\sum_{j \neq i} R_t(i,j),\quad r_t(j \mid i) = R_t(i,j)/\lambda_t(i),
\end{equation*}
where $\lambda_t(i)$ is the \emph{exit rate} governing \emph{when} to jump and $r_t(\cdot \mid i)$ is the \emph{jump distribution} governing \emph{where} to jump. This exit-rate/jump-distribution decomposition is precisely the basis of Gillespie's algorithm~\citep{gillespie1977exact, anderson2007modified} for exact CTMC simulation: first draw a holding time from the exit rate, then choose the next state from the jump distribution. Despite being classical, this structure has not been reflected in the reverse-process parameterization of discrete diffusion models---existing methods still learn the reverse process through proxy quantities such as score ratios, clean-data predictions, or denoising distributions, which implicitly determine $\Rthetamat_t$ as a monolithic object. Inspired by the algorithmic decomposition in Gillespie's algorithm, our starting point is that the training objective itself also decomposes along this timing/direction structure, suggesting a parameterization and training rule that make this separation explicit.

Moreover, the two mainstream forward processes in discrete diffusion are masked (absorbing) and uniform. The masked forward process is dominant in current state-of-the-art methods~\citep{sahoo2024simple, ruette2025gidd}, but its absorbing nature means a token, once unmasked during generation, is frozen for the rest of sampling, limiting diversity and precluding self-correction~\citep{ruette2025gidd}. The uniform forward process avoids this limitation by allowing transitions among all $S$ states, yet it has received comparatively little attention and generally yields weaker results~\citep[e.g.,][]{lou2024discrete, austin2021structured, ruette2025gidd}, among others. GIDD~\citep{ruette2025gidd} partially alleviates these gaps by interpolating between masked and uniform forward processes, but its modification lies on the \emph{forward} (noising) side; the reverse process is still parameterized monolithically through a clean-data predictor. In contrast, we keep the forward process simple and instead modify the \emph{reverse} side, decomposing the reverse rate into its intrinsic timing and direction components.

Building on this perspective, we propose \textbf{Neural CTMC}, which addresses both gaps from the reverse-process side. First, we separately parameterize the exit rate $\lambdatheta_t$ and the jump distribution $r^\theta_t$ for the reverse process via two network heads. We show that the $\theta$-dependent part of the ELBO is equivalent to a reverse-process KL divergence up to a $\theta$-independent constant, and that this KL naturally decomposes as
\begin{equation*}
\KLdiv^{\Poi}(\lambdahat \| \lambdatheta) \;+\; \lambdahat \cdot \KLdiv^{\Cat}(\hat{r} \| r^\theta),
\end{equation*}
a Poisson KL for jump timing and a categorical KL for jump direction. This decomposition is not introduced ad hoc at the architecture level; rather, it is induced by the CTMC path-space objective itself, which suggests learning timing and direction with separate heads. Notably, when specialized to the masked (absorbing) forward process, our objective recovers the MDLM loss exactly, and offers a new perspective on MDLM: its $x_\theta$ predictor is precisely our jump-distribution head $r^\theta$, while the exit rate $\lambdatheta$ collapses to a schedule-determined constant---showing that the decomposed formulation strictly generalizes existing approaches. \textbf{Second, equipped with this decomposed parameterization, Neural CTMC with a pure uniform forward process surpasses several state-of-the-art masked-based methods on the benchmarks we study.}

\paragraph{Contributions.}

\begin{enumerate}[leftmargin=*]
    \item \textbf{A reverse-process parameterization aligned with CTMC structure.} We parameterize the reverse CTMC through separate exit-rate and jump-distribution heads, motivated by the intrinsic timing/direction factorization of CTMC dynamics (Sections~\ref{sec:background} and~\ref{sec:theory}).

    \item \textbf{A decomposition of the reverse-process objective.} We derive a clean form of the path-space ELBO for CTMC-based discrete diffusion, showing that it reduces to a reverse-process KL that splits into a Poisson KL for exit rates and a categorical KL for jump directions. Similarly, its tractable conditional surrogate preserves gradients and minimizers under standard regularity assumptions (Section~\ref{sec:theory}).

    \item \textbf{Strong empirical performance with a pure uniform forward process.} On TinyStories, Neural CTMC achieves generative perplexity $16.36$, compared with $37.60$ for GIDD and $42.66$ for MDLM under our evaluation protocol. On OpenWebText, it achieves the best equal-budget generative perplexity among the methods we evaluate across 16 to 128 sampling steps; it also remains competitive with SEDD despite using $2.6\times$ fewer training tokens (Section~\ref{sec:experiments}). We release open-source checkpoints for our pure uniform-noise discrete diffusion model.
\end{enumerate}

\section{Background}
\label{sec:background}

\subsection{Continuous-Time Markov Chains}
\label{sec:bg_ctmc}

Let $\calS = \{1, \ldots, S\}$ be a finite state space. A CTMC on $\calS$ over $[0, T]$~\citep{norris1997markov} is characterized by a rate matrix $R_t$ with $R_t(i,j) \geq 0$ for $i \neq j$ and $\sum_j R_t(i,j) = 0$.
\begin{definition}
\label{def:forward_transition}
For some start time $s$ and end time $t = s + \Delta$ $(t > s)$ as $\Delta \to 0$, we have
\begin{equation}
\label{eq:forward_transition}
q_{t|s}(j \given i) = \delta_{i,j} + R_t(i,j)\,\Delta + o(\Delta),
\end{equation}
where $R_t$ is called the forward transition rate.
\end{definition}

The Kolmogorov forward equation $\frac{\dd}{\dd t} q_t = q_t R_t$ yields transition probabilities $P_{t|s}$~\citep{campbell2022continuous}. For the uniform-rate forward process with $R_t(i,j) = \beta_t$ for $i \neq j$, the cumulative transition probabilities take a closed form.
\begin{definition}
\label{def:cumulative_transition}
The cumulative transition probabilities of the CTMC are given by
\begin{equation}
\label{eq:cumulative_transition}
q_{t|0}(j \given i) = \mathrm{Cat}(j;\, P_t e_i), \qquad P_t = \alpha_t I + \beta_t \pi \mathbf{1}^\top,
\end{equation}
where $\alpha_t + \beta_t = 1$ with $\alpha_0 = 1, \alpha_T = 0$, $\pi$ is a fixed distribution (e.g., the uniform distribution $\pi = \frac{1}{S}\mathbf{1}$, or a mask distribution $\pi = e_{\texttt{[M]}}$ where $e_{\texttt{[M]}}$ is the one-hot vector for the mask token; $\pi$ may also depend on $t$), and $\mathbf{1}$ is the all-ones vector. As $t \to T$, $q_{t|0}(\cdot \given i) \to \pi$ for all $i$, so the prior distribution is $\pref = \pi$.
\end{definition}

\paragraph{Path measure.} C\`adl\`ag (right-continuous with left limits) paths provide a general framework for stochastic processes with jumps, encompassing L\'evy processes, semimartingales, and others. A CTMC corresponds to piecewise constant c\`adl\`ag paths and is characterized by a path measure $\Qmeas_{x_0}$. A trajectory $W$ drawn from the CTMC can be fully described through its jump times $T_1, \ldots, T_n$ and its state values between jumps $W_0, W_1, \ldots, W_n$, where at jump time $T_k$ the CTMC jumps from state $W_{k-1}$ to $W_k$.

\begin{definition}
\label{def:path_space_measure}
A path space measure $\Qmeas_{x_0}$ assigns probabilities to a drawn trajectory $W$ from time $0$ to $T$ in the sense of
\begin{equation*}
\Qmeas_{x_0}(W \in \dd\omega) := \Pmeas\!\left(W_0 \in \dd\omega_0,\, (T_1, W_{T_1}) \in \dd(t_1, \omega_1),\, \ldots,\, (T_n, W_{T_n}) \in \dd(t_n, \omega_n),\, T_{n+1} \geq T\right),
\end{equation*}
where $\dd\omega_k$ and $\dd t_k$ denote infinitesimal neighborhoods around the points $\omega_k \in \calS$ and $t_k \in [0, T]$.
\end{definition}

Writing $x_0, x_1, \ldots, x_n$ for the state values and $t_1, \ldots, t_n$ for the jump times, with $X_t = x_k$ for $t \in [t_k, t_{k+1})$ (where $t_0 := 0$ and $t_{n+1} := T$), the path measure density takes the following form.

\begin{lemma}
\label{lem:path_measure}
Let $\Qmeas_{x_0}$ denote the path measure induced by a CTMC with rate matrix $R_t$ starting from initial state $x_0$. For a path $\omega$ with jumps at times $0 \leq t_1 < \cdots < t_n < T$ and states $x_0 \to x_1 \to \cdots \to x_n$:
\begin{equation}
\label{eq:path_density}
\dd\Qmeas_{x_0}(\omega) = \underbrace{\prod_{k=1}^n R_{t_k}(x_{k-1}, x_k)}_{\text{jump rates}} \cdot \underbrace{\exp\!\left(-\int_0^T \lambda_t(X_t) \, \dd t\right)}_{\text{survival probability}},
\end{equation}
where $\lambda_t(i) := \sum_{j \neq i} R_t(i,j)$ is the total transition rate out of state $i$.
\end{lemma}
The proof is given in Appendix~\ref{app:proof_path_measure}. Equation~\eqref{eq:path_density} will serve as the starting point of our approach: we treat $\lambda_t$ and $r_t$ as the two natural targets of approximation, and learn them with separate network heads $\lambdatheta_t$ and $r^\theta_t$ via the path-measure objective derived in Section~\ref{sec:theory}.

\subsection{Discrete Diffusion and Existing Parameterizations}
\label{sec:bg_discrete_diffusion}

A CTMC-based discrete diffusion model consists of a \emph{forward} process and a \emph{reverse} (generative) process. The forward process corrupts data $X_0 \sim \pdata$ via a CTMC with rate matrix $R_t$, yielding conditionals $q_{t|0}(x \given x_0) := \Pmeas(X_t = x \given X_0= x_0)$ and marginals $q_t(x) := \E_{x_0 \sim \pdata}[q_{t|0}(x \given x_0)]$, with $q_t \to \pref$ as $t \to T$. A key structural property of the rate matrix is the following decomposition.

\paragraph{Exit rate and jump distribution.} Any rate matrix admits the unique decomposition $R_t(i,j) = \lambda_t(i) \cdot r_t(j \given i)$, where
\begin{equation}
\label{eq:exit_jump}
\lambda_t(i) := \sum_{j \neq i} R_t(i,j) \;\; \text{(exit rate)}, \qquad r_t(j \given i) := \frac{R_t(i,j)}{\lambda_t(i)} \;\; \text{(jump distribution)}.
\end{equation}
$\lambda_t(i)$ governs the holding time and $r_t(\cdot \given i)$ is a categorical distribution over jump destinations. This factorization separates \emph{when} to jump from \emph{where} to jump---a structural property of any CTMC.

\paragraph{Reverse process.} The time-reversal of the forward process is a CTMC with rate matrix~\citep{campbell2022continuous}:
\begin{equation}
\label{eq:reverse_rate}
\Rhatmat_t(i,j) = R_t(j,i) \frac{q_t(j)}{q_t(i)}, \quad i \neq j,
\end{equation}
A neural network parameterizes a reverse CTMC with rate matrix $\Rthetamat_t(i,j)$ starting from $X_T \sim \pref$, with the goal of minimizing an upper bound on the negative log-likelihood $-\E_{x_0 \sim \pdata}[\log p_\theta(x_0)]$.
\label{sec:bg_param}

\paragraph{Existing parameterizations.} Prior work parameterizes $\Rthetamat_t$ through different proxy quantities, each implicitly determining the full reverse rate matrix:
\begin{itemize}[leftmargin=*, itemsep=2pt]
    \item \textbf{Concrete score}~\citep{lou2024discrete}: parameterize $s_\theta(x_t, t)_y \approx \frac{q_t(y)}{q_t(x_t)}$, recover $\Rthetamat_t(i,j) = R_t(j,i) \cdot s_\theta(i,t)_j$. Trained via score entropy.
    \item \textbf{Denoising distribution}~\citep{campbell2022continuous}: parameterize $p_\theta(x_0 \given x_t)$, construct $\Rthetamat_t(i,j) = \sum_{x_0} p_\theta(x_0 \given i) \, R_t(j,i) \, \frac{q_{t|0}(j \given x_0)}{q_{t|0}(i \given x_0)}$.
    \item \textbf{Clean data prediction} ($x_0$-param.)~\citep{austin2021structured, sahoo2024simple, ruette2025gidd}: the model distribution $p_\theta(z_s \given z_t) = q_{t|s}(z_t \given z_s)\, \frac{q_s(z_s \given x_\theta)}{q_t(z_t \given x_\theta)}$, where $x_\theta(z_t, t)$ is a neural network that predicts the clean data.
\end{itemize}
All three strategies determine $\Rthetamat_t(i,j)$ without decomposing it into exit rate and jump distribution---which, as we show, correspond to independent learning objectives.

\section{Related Work}
\label{sec:related}

D3PM~\citep{austin2021structured} and Multinomial Diffusion~\citep{hoogeboom2021argmax} introduced discrete diffusion using transition matrices in discrete time. \citet{campbell2022continuous} extended this to continuous time via CTMCs, deriving the ELBO from path measure ratios. Subsequent work explored different parameterizations of the reverse rate matrix: SEDD~\citep{lou2024discrete} uses a concrete score analogous to continuous score matching~\citep{song2021scorebased, meng2022concrete}, MDLM~\citep{sahoo2024simple} simplifies masked diffusion with $x_0$-prediction, GIDD~\citep{ruette2025gidd} unifies these under a general interpolating framework, and discrete flow matching~\citep{campbell2024generative, gat2024discrete} casts CTMC-based discrete generation in a flow-matching framework with multimodal extensions. Recent concurrent work has also explored simplified masked objectives~\citep{shi2024simplified, ou2025absorbdiscrete}, guided generation in discrete spaces~\citep{nisonoff2024unlocking}, and improved sampling via informed correctors~\citep{zhao2024informed}. On the variational side, the ELBO for continuous diffusion was established via SDE path measures and Girsanov's theorem~\citep{song2021maximum, kingma2021variational}; the discrete analog replaces SDEs with CTMCs and Girsanov's theorem with the Campbell--Mecke formula~\citep{jacobsen2006point}, and \citet{sun2023score} derived a score-based continuous-time objective for discrete diffusion along these lines.

\section{Methodology}
\label{sec:method}

\subsection{Theory}
\label{sec:theory}

We retain the setup from Section~\ref{sec:background}: $\calS = \{1, \ldots, S\}$, forward CTMC with rate matrix $R_t$ corrupting $X_0 \sim \pdata$ into $X_T \sim \pref$, with conditionals $q_{t|0}$ and marginals $q_t$. The true reverse CTMC has rate matrix $\Rhatmat_t(i,j) = R_t(j,i)\, \frac{q_t(j)}{q_t(i)}$. By the decomposition~\eqref{eq:exit_jump}, its dynamics factor as:
\begin{equation}
\lambdahat_t(i) = \sum_{j \neq i} \Rhatmat_t(i,j) \;\; \text{(exit rate)}, \qquad \hat{r}_t(j \given i) = \frac{\Rhatmat_t(i,j)}{\lambdahat_t(i)} \;\; \text{(jump distribution)}.
\end{equation}
Rather than parameterizing a proxy for $\Rthetamat_t(i,j)$ (Section~\ref{sec:bg_param}), we directly parameterize both components via a neural network $\Phi_\theta: \calS \times [0,T] \to \R_{>0} \times \Delta^{S-1}$\footnote{$\Delta^{S-1} = \{p \in \R^S_{\geq 0} : \sum_{j} p_j = 1\}$ denotes the $(S{-}1)$-dimensional probability simplex.}:
\begin{equation}
\label{eq:neural_param}
\Phi_\theta(x_t, t) = \bigl(\lambdatheta_t(x_t), \; r^\theta_t(\cdot \given x_t)\bigr), \qquad \Rthetamat_t(i,j) = \lambdatheta_t(i) \cdot r^\theta_t(j \given i) \;\; (j \neq i).
\end{equation}
Crucially, the variational objective inherits this factorization, decomposing into two subproblems: predicting \emph{when to jump} and \emph{where to jump}. We state the main results here and defer proofs to Appendix~\ref{app:proofs}.

\begin{lemma}
\label{lem:path_ratio}
Let $\Qmeas_{x_0}$ be the forward path measure starting from $x_0$, and $\Pmeas^{\theta}$ the neural reverse path measure. For a path $\omega$ with jumps at times $0 < t_1 < \cdots < t_n < T$ and states $x_0 \to x_1 \to \cdots \to x_n$:
\begin{equation}
\log \frac{\dd\Pmeas^{\theta}}{\dd\Qmeas_{x_0}}(\omega) = \underbrace{\log \pref(x_n)}_{\textnormal{prior}} + \underbrace{\sum_{k=1}^n \log \frac{\Rthetamat_{t_k}(x_k, x_{k-1})}{R_{t_k}(x_{k-1}, x_k)}}_{\textnormal{jump ratio}} + \underbrace{\int_0^T \!\left[\lambdafwd_t(X_t) - \lambdatheta_t(X_t)\right] \dd t}_{\textnormal{survival difference}}.
\end{equation}
\end{lemma}

The proof follows from taking the ratio of the forward and reverse path densities; see Appendix~\ref{app:proofs}. Combining this ratio with importance sampling under $\Qmeas_{x_0}$ and Jensen's inequality yields the following ELBO.

\begin{theorem}
\label{thm:single_elbo}
For any $x_0 \in \calS$:
\begin{equation}
-\log p_\theta(x_0) \leq -\E_{q_{T|0}}[\log \pref(X_T)] + \int_0^T \sum_{i \in \calS} q_{t|0}(i \given x_0) \, \ell_t(i) \, \dd t,
\end{equation}
where $\ell_t(i) := \sum_{j \neq i} \left[R_t(i,j) \log \frac{R_t(i,j)}{\Rthetamat_t(j,i)} - R_t(i,j) + \Rthetamat_t(i,j)\right]$ is the Bregman divergence generated by $\phi(x) = x \log x$.
\end{theorem}

Theorem~\ref{thm:single_elbo} is a conditional bound in which the loss $\ell_t$ couples the forward rate $R_t$ with the reverse parameterization $\Rthetamat_t$. Averaging over $x_0 \sim \pdata$ marginalizes the conditional $q_{t|0}(\cdot\given x_0)$ into $q_t$, after which the $\theta$-dependent part rewrites cleanly as a reverse-process path-space KL up to a $\theta$-independent constant.

\begin{theorem}
\label{thm:data_elbo}
\begin{equation}
\mathcal{E}_{\mathrm{data}}(\theta) := \E_{x_0 \sim \pdata}\!\left[-\E_{q_{T|0}}[\log \pref(X_T)] + \int_0^T \sum_{i \in \calS} q_{t|0}(i \given x_0)\,\ell_t(i)\,\dd t\right].
\end{equation}
\begin{equation}
\E_{x_0 \sim \pdata}[-\log p_\theta(x_0)] \leq \mathcal{E}_{\mathrm{data}}(\theta),
\end{equation}
and
\begin{equation}
\mathcal{E}_{\mathrm{data}}(\theta) = \KLdiv(\hat{\Qmeas} \| \Pmeas^{\theta}) + C,
\end{equation}
where $C$ is independent of $\theta$, $\hat{\Qmeas}$ denotes the path measure of the true reverse CTMC with rate matrix $\Rhatmat_t$, $\Pmeas^{\theta}$ denotes the path measure of the parameterized reverse CTMC with rate matrix $\Rthetamat_t$, and
\begin{equation}
\label{eq:reverse_kl}
\KLdiv(\hat{\Qmeas} \| \Pmeas^{\theta}) = \int_0^T \sum_{i \in \calS} q_t(i) \sum_{j \neq i} \left[\Rhatmat_t(i,j) \log \frac{\Rhatmat_t(i,j)}{\Rthetamat_t(i,j)} - \Rhatmat_t(i,j) + \Rthetamat_t(i,j)\right] \dd t.
\end{equation}
\end{theorem}

Since both rates in~\eqref{eq:reverse_kl} now refer to the same reverse direction, the objective admits a timing-direction decomposition.

\begin{corollary}
\label{cor:decomp}
The reverse process KL decomposes as:
\begin{equation}
\label{eq:decomp}
\KLdiv(\hat{\Qmeas} \| \Pmeas^{\theta}) = \int_0^T \sum_{i \in \calS} q_t(i) \left[\underbrace{\KLdiv^{\Poi}\!\bigl(\lambdahat_t(i) \| \lambdatheta_t(i)\bigr)}_{\textnormal{when to jump}} + \lambdahat_t(i) \cdot \underbrace{\KLdiv^{\Cat}\!\bigl(\hat{r}_t(\cdot \given i) \| r^\theta_t(\cdot \given i)\bigr)}_{\textnormal{where to jump}}\right] \dd t,
\end{equation}
where $\KLdiv^{\Poi}(\lambda \| \lambda^\theta) := \KLdiv\!\bigl(\Poi(\lambda)\,\|\,\Poi(\lambda^\theta)\bigr) = \lambda \log \frac{\lambda}{\lambda^\theta} - \lambda + \lambda^\theta$ is the Poisson KL and $\KLdiv^{\Cat}(\hat{r} \| r^\theta) := \sum_{j \neq i} \hat{r}(j \given i) \log \frac{\hat{r}(j \given i)}{r^\theta(j \given i)}$ is the categorical KL.
\end{corollary}

The exit-rate head is trained by a Poisson KL and the jump-distribution head by a categorical KL weighted by $\lambdahat_t(i)$, separating \emph{whether to move} from \emph{where to move next}.

\begin{proposition}
\label{prop:pointwise_opt}
For fixed $t$ and $i$, define
\begin{equation}
g_{t,i}(\lambda, r) := \KLdiv^{\Poi}\!\bigl(\lambdahat_t(i) \| \lambda\bigr) + \lambdahat_t(i)\,\KLdiv^{\Cat}\!\bigl(\hat{r}_t(\cdot \given i) \| r\bigr).
\end{equation}
If $\lambdahat_t(i) > 0$, then $g_{t,i}(\lambda, r) \ge 0$ with equality if and only if $\lambda = \lambdahat_t(i)$ and $r = \hat{r}_t(\cdot \given i)$. If $\lambdahat_t(i)=0$, the unique minimizer in $\lambda$ is $\lambda=0$ and the choice of $r$ is irrelevant. The training objective~\eqref{eq:decomp} can be estimated via Monte Carlo sampling. Using the cross-entropy identity $\KLdiv^{\Cat} = \mathrm{CE} - H$:
\begin{equation}
\label{eq:mc_loss}
\mathcal{L}(\theta) = \E_{t \sim \mathcal{U}(0,T),\, x_t \sim q_t}\!\left[\KLdiv^{\Poi}\!\bigl(\lambdahat_t(x_t) \| \lambdatheta_t(x_t)\bigr) + \lambdahat_t(x_t) \cdot \mathrm{CE}\!\bigl(\hat{r}_t(\cdot \given x_t),\, r^\theta_t(\cdot \given x_t)\bigr)\right] + \mathrm{const},
\end{equation}
where $\mathrm{CE}(\hat{r}, r^\theta) := -\sum_{j \neq i} \hat{r}(j \given i) \log r^\theta(j \given i)$ is the cross-entropy, and the constant absorbs the entropy of $\hat{r}_t$ (independent of $\theta$).
\end{proposition}

Since $q_t(x) = \E_{x_0}[q_{t|0}(x \given x_0)]$ is intractable, we replace marginal quantities by conditional counterparts via the expectation transformation (Lemma~\ref{lem:expectation_transform} in Appendix~\ref{app:proofs}), yielding the conditional formulation.

\begin{proposition}
\label{prop:conditional_loss}
For a fixed data point $x_0 \in \calS$, define the conditional reverse rates:
\begin{equation}
\hat{\lambda}_{t|0}(i \given x_0) := \sum_{j \neq i} R_t(j, i) \frac{q_{t|0}(j \given x_0)}{q_{t|0}(i \given x_0)}, \qquad
\hat{r}_{t|0}(j \given i, x_0) := \frac{R_t(j, i) \frac{q_{t|0}(j \given x_0)}{q_{t|0}(i \given x_0)}}{\hat{\lambda}_{t|0}(i \given x_0)}.
\end{equation}
Then the training loss~\eqref{eq:mc_loss} can be equivalently written as:
\begin{equation}
\label{eq:conditional_loss}
\begin{aligned}
\mathcal{L}(\theta) = \E_{\substack{t \sim \mathcal{U}(0,T),\, x_0 \sim \pdata \\ x_t \sim q_{t|0}(\cdot \given x_0)}}\!\Bigl[
&- \sum_{j \neq x_t} R_t(j, x_t) \frac{q_{t|0}(j \given x_0)}{q_{t|0}(x_t \given x_0)} \log \bigl(\lambdatheta_t(x_t)\, r^\theta_t(j \given x_t)\bigr) \\
&- \hat{\lambda}_{t|0}(x_t \given x_0) + \lambdatheta_t(x_t)\Bigr] + \mathrm{const}.
\end{aligned}
\end{equation}
\end{proposition}
The proof is given in Appendix~\ref{app:proofs}.

\begin{theorem}
\label{thm:conditional_kl}
Define the \emph{conditional reverse rate} given clean data $x_0$:
\begin{equation}
\label{eq:conditional_reverse_rate}
\Rhatmat_t(i, j \given x_0) := R_t(j, i)\, \frac{q_{t|0}(j \given x_0)}{q_{t|0}(i \given x_0)}, \quad i \neq j,
\end{equation}
and the Poisson KL density $f(r, c) := r \log \frac{r}{c} - r + c$. The \emph{conditional KL training loss}
\begin{equation}
\label{eq:L_kl}
\mathcal{L}_{\mathrm{KL}}(\theta) := \int_0^T \E_{x_0 \sim \pdata,\; x_t \sim q_{t|0}(\cdot \given x_0)} \!\left[\sum_{j \neq x_t} f\!\left(\Rhatmat_t(x_t, j \given x_0),\; \Rthetamat_t(x_t, j)\right)\right] \dd t
\end{equation}
is a tractable upper bound on $\KLdiv(\hat{\Qmeas} \| \Pmeas^{\theta})$:
\begin{equation}
\label{eq:kl_upper_bound}
\E_{\pdata}[-\log p_\theta(x_0)] \leq \mathcal{L}_{\mathrm{KL}}(\theta) + C_{\mathrm{total}},
\end{equation}
where $C_{\mathrm{total}}$ is independent of $\theta$. Moreover, $\mathcal{L}_{\mathrm{KL}}(\theta)$ equals the loss in~\eqref{eq:conditional_loss} up to a $\theta$-independent constant.
\end{theorem}

Theorem~\ref{thm:conditional_kl} establishes a tractable objective value upper bound. The next proposition strengthens this by showing that, under standard regularity assumptions, the conditional surrogate differs from the marginal reverse-process KL only by a $\theta$-independent gap.

\begin{proposition}
\label{prop:exact_gradient}
Assume: (i) forward quantities ($q_t$, $q_{t|0}$, $p(x_0\given x_t)$) are independent of $\theta$; (ii) $\Rthetamat_t(i,j) > 0$ and differentiable in $\theta$ for all $i\neq j$; and (iii) differentiation can be exchanged with expectation/integration. Then there exists a $\theta$-independent constant $C_{\mathrm{gap}} \ge 0$ such that
\begin{equation*}
\mathcal{L}_{\mathrm{KL}}(\theta) = \KLdiv(\hat{\Qmeas}\|\Pmeas^\theta) + C_{\mathrm{gap}},
\end{equation*}
hence
\begin{equation*}
\nabla_\theta \mathcal{L}_{\mathrm{KL}}(\theta) = \nabla_\theta \KLdiv(\hat{\Qmeas}\|\Pmeas^\theta).
\end{equation*}
In particular,
\begin{equation*}
\arg\min_\theta \mathcal{L}_{\mathrm{KL}}(\theta) = \arg\min_\theta \KLdiv(\hat{\Qmeas}\|\Pmeas^\theta),
\end{equation*}
and every stationary point of one objective is a stationary point of the other.
\end{proposition}
Further discussion of the upper-bound interpretation and the information-theoretic form of $C_{\mathrm{gap}}$ is deferred to Appendix~\ref{app:proof_exact_gradient}.

\begin{corollary}
\label{cor:L_cond}
Defining $K(a) := a(\log a - 1)$, the losses~\eqref{eq:conditional_loss} and~\eqref{eq:L_kl} are both equivalent to:
\begin{equation}
\label{eq:L_cond}
\begin{aligned}
\mathcal{L}_{\mathrm{cond}}(\theta) = \E_{t,\, x_0,\, i}\!\Bigl[
&\sum_{j \neq i} \lambdatheta_t(i) r^\theta_t(j \given i)
- \sum_{j \neq i} R_t(j,i) \frac{q_{t|0}(j \given x_0)}{q_{t|0}(i \given x_0)} \log \frac{\lambdatheta_t(i)\, r^\theta_t(j \given i)}{R_t(j,i)} \\
&+ \sum_{j \neq i} R_t(j,i)\, K\!\!\left(\frac{q_{t|0}(j \given x_0)}{q_{t|0}(i \given x_0)}\right)\Bigr].
\end{aligned}
\end{equation}
The first two terms form the per-pair Poisson KL density $f\!\left(\Rhatmat_t(i,j \given x_0),\, \Rthetamat_t(i,j)\right) \geq 0$; the last term is $\theta$-independent and can be precomputed. This avoids large numerical cancellations between the $-\hat{\lambda}_{t|0}\log\lambdatheta_t$ and $+\lambdatheta_t$ terms in~\eqref{eq:conditional_loss} when $\hat{\lambda}_{t|0} \gg 1$ near $t \approx T$.
\end{corollary}

Through the above derivations, we obtain two equivalent training losses: $\mathcal{L}_{\mathrm{KL}}(\theta)$~\eqref{eq:L_kl} and $\mathcal{L}_{\mathrm{cond}}(\theta)$~\eqref{eq:L_cond}. Under the assumptions stated above, they differ only by a $\theta$-independent constant and therefore induce the same first-order optimization problem.

\begin{corollary}[Equivalence to MDLM]
\label{cor:mdlm}
If the forward process is the masked (absorbing) process on $\calS = \mathcal{V} \cup \{m\}$ with rate $R_t(i,j) = -\frac{\alpha_t'}{\alpha_t}(\boldsymbol{e}_j^\top \boldsymbol{e}_m)\indicator_{\{i \neq m\}}$, then for any valid noise schedule $\alpha_t$, the loss~\eqref{eq:L_cond} reduces exactly to the MDLM cross-entropy loss:
\begin{equation}
\label{eq:mdlm_equiv}
\mathcal{L}_{\mathrm{cond}}(\theta) = \E_{t,\, x_0,\, x_t}\!\left[\frac{\alpha_t'}{1-\alpha_t}\,\boldsymbol{x}_0^\top \log x_\theta \cdot \indicator_{\{x_t = m\}}\right],
\end{equation}
\end{corollary}
The proof is given in Appendix~\ref{app:proof_mdlm}. This shows that MDLM is a special case of Neural CTMC's objective restricted to the masked forward process; the uniform forward process used in our experiments is a strict generalization. \textbf{In the masked case, $r^\theta_t$ reduces to the $x_\theta$ predictor in MDLM, while $\lambdatheta_t(m) = -\frac{\alpha_t'}{1-\alpha_t}$ becomes a constant determined by the noise schedule.}

\subsection{Algorithm}
\label{sec:algorithm}

\textbf{Training algorithm.} Following the derivation in Section~\ref{sec:theory}, we form the conditional reverse targets $\hat{R}_t(x_t,\,j\given x_0)$ and $\hat{\lambda}_{t|0}(x_t\given x_0)$ from the forward conditionals~\eqref{eq:cumulative_transition} and minimize either $\mathcal{L}_{\mathrm{KL}}$~\eqref{eq:L_kl} or $\mathcal{L}_{\mathrm{cond}}$~\eqref{eq:L_cond}; the full procedure is summarized in Algorithm~\ref{alg:training}.

\begin{algorithm}[H]
\caption{Neural CTMC Training}
\label{alg:training}
\begin{algorithmic}[1]
\WHILE{not converged}
\STATE Sample $x_0 \sim \pdata$, $t \sim \mathcal{U}(\epsilon, T-\epsilon)$; form $q_{t|0}(\cdot \given x_0)$ via~\eqref{eq:cumulative_transition}
\STATE Sample $x_t \sim q_{t|0}(\cdot \given x_0)$ and form the conditional targets $\hat{R}_t(x_t, j \given x_0)$ and $\hat{\lambda}_{t|0}(x_t \given x_0)$
\STATE Compute $\bigl(\lambdatheta_t(x_t),\, r^\theta_t(\cdot \given x_t)\bigr) = \Phi_\theta(x_t, t)$, evaluate $\mathcal{L}_{\mathrm{KL}}(\theta)$~\eqref{eq:L_kl} or $\mathcal{L}_{\mathrm{cond}}(\theta)$~\eqref{eq:L_cond}, and update $\theta \gets \theta - \eta\,\nabla_\theta \mathcal{L}$
\ENDWHILE
\end{algorithmic}
\end{algorithm}

\textbf{Sampling algorithms.} Given a trained model $\Phi_\theta$, we generate samples by discretizing the reverse CTMC over $[0, T]$ with step size $\tau = T/N$. We consider two sampling schemes. \textbf{$\tau$-Leaping}~\citep{gillespie1977exact} (Algorithm~\ref{alg:tau_leaping}) applies the jump-chain/holding-time view in parallel across sequence positions: each position draws a holding time from its learned exit rate and, if the holding time falls within the current step, jumps according to its learned jump distribution. This naturally exploits the exit-rate/jump-direction decomposition and allows multiple positions to change in a single simulation step. \textbf{Euler} (Algorithm~\ref{alg:euler}) constructs a one-step categorical transition from the off-diagonal probabilities $p_j = \lambda^\theta_t(x_t)\cdot r^\theta_t(j\given x_t)\cdot\tau$ for $j \neq x_t$, with the remaining mass assigned to staying at $x_t$. Pseudocode for both samplers is given in Appendix~\ref{app:exp_details}.

\section{Experiments}
\label{sec:experiments}

We evaluate Neural CTMC on discrete image and language generation tasks. All models use the uniform forward process (Section~\ref{sec:algorithm}) with $\alpha_t=1-t$ and $\beta_t=t$. For a fair comparison, Neural CTMC adopts the same DiT~\citep{peebles2023scalable} backbone and parameter count as the mainstream baselines; full hyperparameters are given in Appendix~\ref{app:exp_details}.

\textbf{Image Generation:}
\label{sec:exp_image}
We train Neural CTMC on the MNIST dataset, flattened to sequences of length $28\times 28$ over $\calS = \{0,1,\ldots,255\}$. Using the same DiT backbone, Figure~\ref{fig:mnist} shows 100 conditional samples (10 per digit class) generated with 128 sampling steps at epoch 80. Both Euler and $\tau$-leaping produce clearly recognizable digits, confirming that the decoupled parameterization learns meaningful structure even on image data treated as a discrete sequence.

\textbf{Language Modeling:}
\label{sec:exp_language}
We compare Neural CTMC against SEDD~\citep{lou2024discrete}, MDLM~\citep{sahoo2024simple}, and GIDD~\citep{ruette2025gidd} on TinyStories~\citep{eldan2023tinystories} and OpenWebText~\citep{radford2019language} (OWT). 
To ensure a fair comparison, Neural CTMC and all baselines (SEDD, MDLM, GIDD) use an identical backbone architecture (12-layer DiT-style Transformer, 768 hidden dim, 12 heads, max\_len=512, $\sim$163M parameters); the only difference across methods is the parameterization of the reverse process and its associated loss.
For GIDD we report results with $p_{\text{unif}} \in \{0.0, 0.1, 0.2\}$, where $p_{\text{unif}}=0.0$ corresponds to a pure mask process. For evaluation, we draw 1024 unconditional samples from each model and score them with a pretrained \textbf{Gemma2-9B} model to obtain \emph{generative perplexity} (PPL); each method is run with multiple sampling seeds and we report the best PPL across seeds.

\begingroup
\setlength{\intextsep}{2pt}
\begin{figure}[H]
\centering
\begin{subfigure}[t]{0.48\textwidth}
    \centering
    \includegraphics[width=\linewidth]{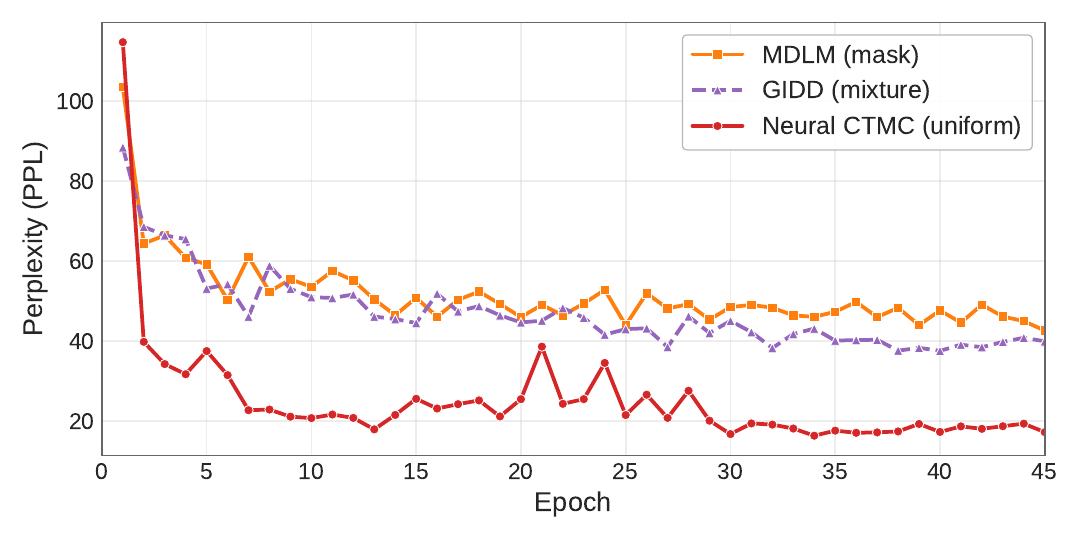}
    \caption{TinyStories: PPL over training epochs (50 sampling steps).}
    \label{fig:tinystory_ppl}
\end{subfigure}
\hfill
\begin{subfigure}[t]{0.48\textwidth}
    \centering
    \includegraphics[width=\linewidth]{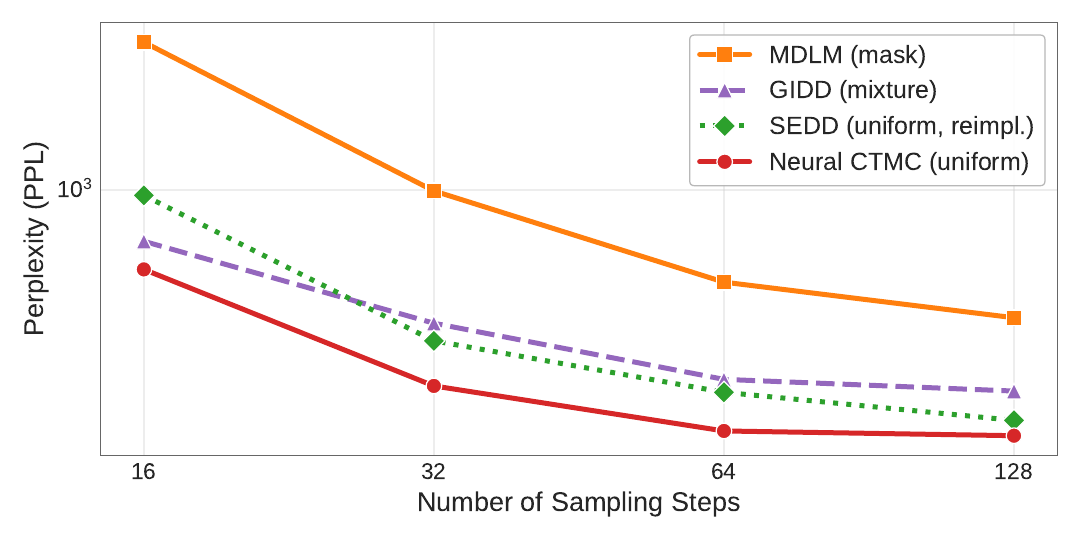}
    \caption{OpenWebText: PPL vs.\ number of sampling steps (log scale).}
    \label{fig:owt_ppl}
\end{subfigure}
\caption{\textbf{Sample quality comparison.} (a)~On TinyStories, Neural CTMC converges to substantially lower perplexity than MDLM and GIDD under identical training conditions. (b)~On OpenWebText, Neural CTMC performs best among the equal-budget baselines across the step counts we evaluate.}
\label{fig:ppl_curves}
\end{figure}
\endgroup

\textbf{TinyStories.} We train all methods from their official source code under the same architecture and training budget. Figure~\ref{fig:ppl_curves}(a) plots generative PPL against training epoch with 50 sampling steps. Neural CTMC converges faster and reaches substantially lower perplexity than the baselines. Across the 45-epoch training window, Neural CTMC attains a best PPL of $16.36$, compared with $37.60$ for GIDD ($p_{\text{unif}}=0.0$) and $42.66$ for MDLM.

\vspace{-0.5em}
\begin{table}[H]
\centering
\caption{Generative perplexity ($\downarrow$) on OpenWebText for varying numbers of sampling steps. All methods use the same backbone architecture and parameter scale. \textbf{Bold}: best overall per column; \underline{underline}: best among equal-budget (262B) methods.}
\label{tab:owt_ppl}
\small
\begin{tabular}{clccrrrr}
\toprule
Type & Method & Train Toks & Max Len & 16 & 32 & 64 & 128 \\
\midrule
\multirow{3}{*}{mask}    & SEDD~\citep{lou2024discrete} & 682B & 1024 & $825.5$ & $337.9$ & $\textbf{186.5}$ & $\textbf{127.2}$ \\
                         & MDLM~\citep{sahoo2024simple} & 262B & 1024 & $1432.8$ & $553.7$ & $301.6$ & $210.5$ \\
                         & GIDD ($p_{\text{unif}}\!=\!0.0$)~\citep{ruette2025gidd} & 262B & 512 & $2773.1$ & $993.7$ & $529.6$ & $414.3$ \\
\midrule
\multirow{2}{*}{mixture} & GIDD ($p_{\text{unif}}\!=\!0.1$)~\citep{ruette2025gidd} & 262B & 512 & $702.0$ & $398.9$ & $270.8$ & $249.8$ \\
                         & GIDD ($p_{\text{unif}}\!=\!0.2$)~\citep{ruette2025gidd} & 262B & 512 & $770.4$ & $430.1$ & $344.3$ & $293.0$ \\
\midrule
\multirow{2}{*}{uniform} & SEDD (reimpl.)$^{\dagger}$~\citep{lou2024discrete} & 262B & 512 & $963.9$ & $353.2$ & $247.7$ & $204.1$ \\
                         & Neural CTMC (ours) & 262B & 512 & $\underline{\textbf{578.3}}$ & $\underline{\textbf{258.8}}$ & $\underline{189.7}$ & $\underline{183.6}$ \\
\bottomrule
\end{tabular}
\par\vspace{2pt}{\footnotesize Note: (1) The losses of SEDD (mask) and GIDD ($p_{\text{unif}}\!=\!0.0$) are equivalent to the MDLM loss. (2) For a fair comparison across max-len settings, models with max-len=1024 generate samples of length 1024, and PPL is computed on the first 512 tokens. ($\dagger$) SEDD (reimpl.) is retrained from scratch using the official open-source SEDD codebase under the same hyperparameters, max-len, and training token budget as Neural CTMC.\par}
\end{table}
\vspace{-0.5em}

\textbf{OpenWebText.} To allow the baselines to exhibit their best performance, for OWT we use the official open-source released checkpoints and configurations of SEDD (mask), MDLM, and GIDD. These checkpoints share the same backbone architecture, parameter scale, tokenization (GPT-2), and evaluation pipeline (Gemma2-9B scoring) as Neural CTMC, representing our best effort at a unified comparison. MDLM and GIDD are trained on 262B tokens, exactly matching our training budget, while SEDD uses 682B ($2.6\times$ more) and is marked separately. Table~\ref{tab:owt_ppl} and Figure~\ref{fig:ppl_curves}(b) report PPL as a function of sampling steps: among the equal-budget methods (262B tokens), Neural CTMC outperforms MDLM and GIDD at every step count from 16 to 128, with the largest gain at 32 steps where Neural CTMC ($258.8$) improves over MDLM ($553.7$) by $2.1\times$ and GIDD ($398.9$) by $1.5\times$; under matched training conditions, Neural CTMC outperforms SEDD (uniform) at every step count from 16 to 128; compared with SEDD (mask), Neural CTMC is stronger at low step counts (16 and 32), essentially ties at 64 steps ($189.7$ vs.\ $186.5$), and is weaker at 128 steps ($183.6$ vs.\ $127.2$), likely reflecting SEDD (mask)'s larger training budget and longer max-len.

\textbf{Scaling law.} To assess whether the gains of Neural CTMC persist as compute grows, we conduct a scaling study on OpenWebText along three axes---model size $\{$tiny (80.1M), small (169.7M), base (424.5M)$\}$, token budget $\{$100M, 300M, 1B$\}$, and total training FLOPs---with three seeds per configuration. As shown in Figure~\ref{fig:scaling_law}, validation loss decreases monotonically with both parameter count and training tokens, and a clean compute-optimal frontier emerges, with the best run (\texttt{base} at 1B tokens) reaching $4.150$. The fitted power-law exponent on the frontier ($\alpha\!\approx\!0.16$, $R^2\!\approx\!0.93$) is in line with mature autoregressive and diffusion-based language models, suggesting that the proposed parameterization scales gracefully and that further gains can be expected at larger compute budgets; see Appendix~\ref{app:exp_details} for details.

\textbf{Self-correction.} A key advantage of the uniform forward process is that, unlike masked diffusion, generated tokens remain editable during sampling, enabling lightweight post-hoc self-correction. We compare a self-correction (Algorithm~\ref{alg:self_correction}) against the GIDD self-correction baselines on identical raw OpenWebText samples, with outputs scored by GPT-5.2 on five quality axes. As shown in Table~\ref{tab:self_correction_quality}, Neural CTMC self-correction at temperature $\tau_{\mathrm{sc}}=0.1$ achieves the best Grammaticality ($+12.4\%$), Factuality ($+6.2\%$), and Creativity ($+8.0\%$) among all variants, while GIDD ($p_u=0.2$) leads on Clarity and Writing style; see Appendix~\ref{app:exp_details} for the full protocol.

\section{Conclusion}
\label{sec:conclusion}

We present Neural CTMC, a discrete diffusion framework that parameterizes the reverse CTMC through separate exit-rate and jump-distribution heads, aligning the model with the intrinsic timing/direction structure of CTMC dynamics. The resulting ELBO factorizes into a Poisson KL for timing and a categorical KL for direction, and admits a tractable conditional loss that is gradient-equivalent under standard regularity assumptions. Empirically, this decomposition makes the pure uniform forward process competitive on the benchmarks we study: on TinyStories, Neural CTMC reaches generative perplexity $16.36$, outperforming the baselines under the same protocol; on OpenWebText, it attains the best equal-budget perplexity across 16 to 128 sampling steps. A scaling study further shows that validation loss decreases monotonically with both model size and tokens, with a clean power-law frontier ($\alpha\!\approx\!0.16$, $R^2\!\approx\!0.93$). Finally, the editable nature of the uniform process enables a lightweight self-correction step that achieves the best Grammaticality, Factuality, and Creativity scores among all variants we evaluate. The exit-rate/jump-distribution view, together with our open-source uniform-noise checkpoints, offers the community a principled and reusable foundation for future work on scaling, post-training, and controllable generation in discrete diffusion.

\section{Limitations}
\label{sec:limitations}

Our results have several limitations, most stemming from compute constraints rather than the framework itself. \emph{(i) Scale.} The largest model we train is 424.5M parameters with a 512-token context; whether the gains persist at billion-parameter scale and longer contexts is left to future work, though the trend in Appendix~\ref{app:exp_details} (Figure~\ref{fig:scaling_law}) is encouraging. \emph{(ii) Post-training and RL.} We only consider pre-training and do not study instruction tuning, preference optimization, or RL-style fine-tuning on top of the learned exit-rate / jump-direction model. \emph{(iii) Forward processes.} The framework supports general mixture forward processes, but we only train and evaluate the pure uniform process.

\bibliographystyle{plainnat}
\bibliography{references}

@article{eldan2023tinystories,
  title={TinyStories: How Small Can Language Models Be and Still Speak Coherent English?},
  author={Eldan, Ronen and Li, Yuanzhi},
  journal={arXiv preprint arXiv:2305.07759},
  year={2023}
}

@article{campbell2022continuous,
  title={A Continuous Time Framework for Discrete Denoising Models},
  author={Campbell, Andrew and Benton, Joe and De Bortoli, Valentin and Rainforth, Tom and Deligiannidis, George and Doucet, Arnaud},
  journal={NeurIPS},
  year={2022}
}

@article{hyvarinen2005estimation,
  title={Estimation of Non-Normalized Statistical Models by Score Matching},
  author={Hyv{\"a}rinen, Aapo},
  journal={Journal of Machine Learning Research},
  volume={6},
  pages={695--709},
  year={2005}
}

@inproceedings{song2021maximum,
  title={Maximum Likelihood Training of Score-Based Diffusion Models},
  author={Song, Yang and Durkan, Conor and Murray, Iain and Ermon, Stefano},
  booktitle={Advances in Neural Information Processing Systems},
  year={2021}
}

@inproceedings{peebles2023scalable,
  title={Scalable Diffusion Models with Transformers},
  author={Peebles, William and Xie, Saining},
  booktitle={International Conference on Computer Vision},
  year={2023}
}

@inproceedings{song2021scorebased,
  title={Score-Based Generative Modeling through Stochastic Differential Equations},
  author={Song, Yang and Sohl-Dickstein, Jascha and Kingma, Diederik P and Kumar, Abhishek and Ermon, Stefano and Poole, Ben},
  booktitle={International Conference on Learning Representations},
  year={2021}
}

@book{jacobsen2006point,
  title={Point Process Theory and Applications: Marked Point and Piecewise Deterministic Processes},
  author={Jacobsen, Martin},
  publisher={Birkh{\"a}user},
  year={2006}
}

@article{kingma2021variational,
  title={Variational Diffusion Models},
  author={Kingma, Diederik P and Salimans, Tim and Poole, Ben and Ho, Jonathan},
  journal={Advances in Neural Information Processing Systems},
  volume={34},
  pages={21696--21707},
  year={2021}
}

@inproceedings{sun2023score,
  title={Score-based Continuous-time Discrete Diffusion Models},
  author={Sun, Haoran and Yu, Lijun and Dai, Bo and Schuurmans, Dale and Dai, Hanjun},
  booktitle={International Conference on Learning Representations},
  year={2023}
}

@inproceedings{lou2024discrete,
  title={Discrete Diffusion Modeling by Estimating the Ratios of the Data Distribution},
  author={Lou, Aaron and Meng, Chenlin and Ermon, Stefano},
  booktitle={International Conference on Machine Learning},
  year={2024}
}

@inproceedings{sahoo2024simple,
  title={Simple and Effective Masked Diffusion Language Models},
  author={Sahoo, Subham Sekhar and Arriola, Marianne and Schiff, Yair and Gokaslan, Aaron and Marroquin, Edgar and Chiu, Justin T. and Rush, Alexander M. and Kuleshov, Volodymyr},
  booktitle={Advances in Neural Information Processing Systems},
  volume={37},
  year={2024}
}

@inproceedings{ruette2025gidd,
  title={Generalized Interpolating Discrete Diffusion},
  author={von R{\"u}tte, Dimitri and Fluri, Janis and Ding, Yuhui and Orvieto, Antonio and Sch{\"o}lkopf, Bernhard and Hofmann, Thomas},
  booktitle={International Conference on Machine Learning},
  year={2025}
}

@inproceedings{li2022diffusionlm,
  title={Diffusion-{LM} Improves Controllable Text Generation},
  author={Li, Xiang Lisa and Thickstun, John and Gulrajani, Ishaan and Liang, Percy and Hashimoto, Tatsunori B.},
  booktitle={Advances in Neural Information Processing Systems},
  volume={35},
  year={2022}
}

@article{alamdari2023protein,
  title={Protein Generation with Evolutionary Diffusion: Sequence is All You Need},
  author={Alamdari, Sarah and Thakkar, Nitya and van den Berg, Rianne and Lu, Alex X. and Fusi, Nicolo and Amini, Ava P. and Yang, Kevin K.},
  journal={bioRxiv},
  year={2023}
}

@inproceedings{avdeyev2023dirichlet,
  title={Dirichlet Diffusion Score Model for Biological Sequence Generation},
  author={Avdeyev, Pavel and Shi, Chenlai and Tan, Yuhao and Dudnyk, Kseniia and Zhou, Jian},
  booktitle={International Conference on Machine Learning},
  year={2023}
}

@inproceedings{vignac2023digress,
  title={Di{G}ress: Discrete Denoising Diffusion for Graph Generation},
  author={Vignac, Clement and Krawczuk, Igor and Siraudin, Antoine and Wang, Bohan and Cevher, Volkan and Frossard, Pascal},
  booktitle={International Conference on Learning Representations},
  year={2023}
}

@inproceedings{sohl2015deep,
  title={Deep Unsupervised Learning using Nonequilibrium Thermodynamics},
  author={Sohl-Dickstein, Jascha and Weiss, Eric and Maheswaranathan, Niru and Ganguli, Surya},
  booktitle={International Conference on Machine Learning},
  pages={2256--2265},
  year={2015}
}

@article{austin2021structured,
  title={Structured denoising diffusion models in discrete state-spaces},
  author={Austin, Jacob and Johnson, Daniel D and Ho, Jonathan and Tarlow, Daniel and Van Den Berg, Rianne},
  journal={Advances in Neural Information Processing Systems},
  volume={34},
  pages={17981--17993},
  year={2021}
}

@article{ho2020denoising,
  title={Denoising diffusion probabilistic models},
  author={Ho, Jonathan and Jain, Ajay and Abbeel, Pieter},
  journal={Advances in neural information processing systems},
  volume={33},
  pages={6840--6851},
  year={2020}
}

@inproceedings{shi2024simplified,
  title={Simplified and Generalized Masked Diffusion for Discrete Data},
  author={Shi, Jiaxin and Han, Kehang and Wang, Zhe and Doucet, Arnaud and Titsias, Michalis K.},
  booktitle={Advances in Neural Information Processing Systems},
  volume={37},
  year={2024}
}

@article{gat2024discrete,
  title={Discrete Flow Matching},
  author={Gat, Itai and Remez, Tal and Shaul, Neta and Kreuk, Felix and Chen, Ricky T. Q. and Synnaeve, Gabriel and Adi, Yossi and Lipman, Yaron},
  journal={arXiv preprint arXiv:2407.15595},
  year={2024}
}

@inproceedings{campbell2024generative,
  title={Generative Flows on Discrete State-Spaces: Enabling Multimodal Flows with Applications to Protein Co-Design},
  author={Campbell, Andrew and Yim, Jason and Barzilay, Regina and Rainforth, Tom and Jaakkola, Tommi},
  booktitle={International Conference on Machine Learning},
  year={2024}
}

@inproceedings{zhao2024informed,
  title={Informed Correctors for Discrete Diffusion Models},
  author={Zhao, Yixiu and Shi, Jiaxin and Chen, Feng and Druckmann, Shaul and Mackey, Lester and Linderman, Scott},
  booktitle={Advances in Neural Information Processing Systems},
  year={2025}
}

@inproceedings{gulrajani2024likelihood,
  title={Likelihood-Based Diffusion Language Models},
  author={Gulrajani, Ishaan and Hashimoto, Tatsunori B.},
  booktitle={Advances in Neural Information Processing Systems},
  volume={36},
  year={2023}
}

@article{hoogeboom2021argmax,
  title={Argmax Flows and Multinomial Diffusion: Learning Categorical Distributions},
  author={Hoogeboom, Emiel and Nielsen, Didrik and Jaini, Priyank and Forr{\'e}, Patrick and Welling, Max},
  journal={Advances in Neural Information Processing Systems},
  volume={34},
  year={2021}
}

@inproceedings{meng2022concrete,
  title={Concrete Score Matching: Generalized Score Matching for Discrete Data},
  author={Meng, Chenlin and Choi, Kristy and Song, Jiaming and Ermon, Stefano},
  booktitle={Advances in Neural Information Processing Systems},
  volume={35},
  year={2022}
}

@inproceedings{ou2025absorbdiscrete,
  title={Your Absorbing Discrete Diffusion Secretly Models the Conditional Distributions of Clean Data},
  author={Ou, Jingyang and Nie, Shen and Xue, Kaiwen and Zhu, Fengqi and Sun, Jiacheng and Li, Zhenguo and Li, Chongxuan},
  booktitle={International Conference on Learning Representations},
  year={2025}
}

@inproceedings{nisonoff2024unlocking,
  title={Unlocking Guidance for Discrete State-Space Diffusion and Flow Models},
  author={Nisonoff, Hunter and Xiong, Junhao and Allenspach, Stephan and Listgarten, Jennifer},
  booktitle={International Conference on Learning Representations},
  year={2025}
}

@article{radford2019language,
  title={Language Models are Unsupervised Multitask Learners},
  author={Radford, Alec and Wu, Jeffrey and Child, Rewon and Luan, David and Amodei, Dario and Sutskever, Ilya},
  journal={OpenAI Blog},
  year={2019}
}

@article{gillespie1977exact,
  title={Exact Stochastic Simulation of Coupled Chemical Reactions},
  author={Gillespie, Daniel T.},
  journal={The Journal of Physical Chemistry},
  volume={81},
  number={25},
  pages={2340--2361},
  year={1977}
}

@article{anderson2007modified,
  title={A Modified Next Reaction Method for Simulating Chemical Systems with Time Dependent Propensities and Delays},
  author={Anderson, David F.},
  journal={The Journal of Chemical Physics},
  volume={127},
  number={21},
  year={2007}
}

@book{norris1997markov,
  title={Markov Chains},
  author={Norris, James R.},
  publisher={Cambridge University Press},
  year={1997}
}


\newpage
\appendix

{\huge \bfseries Appendix}


\section*{Contents}
\startcontents[appendix]
\printcontents[appendix]{}{1}{\setcounter{tocdepth}{2}}

\vspace{1.5em}

\section{Mathematical Preliminaries}
\label{app:preliminaries}

We collect here the full definitions of continuous-time Markov chains and point processes used throughout the paper.

\subsection{Continuous-Time Markov Chains}

\begin{definition}
Let $\calS = \{1, \ldots, S\}$ be a finite state space and $T > 0$ a time horizon. A \emph{continuous-time Markov chain} (CTMC) on $\calS$ over $[0, T]$ is a stochastic process $\{X_t\}_{t \in [0,T]}$ with the Markov property:
\begin{equation}
\Pmeas(X_{t+s} = j \given X_t = i, \{X_u\}_{u < t}) = \Pmeas(X_{t+s} = j \given X_t = i).
\end{equation}
\end{definition}

\begin{definition}
A CTMC is characterized by a \emph{rate matrix} $R_t: \calS \times \calS \to \R$ satisfying:
(i) $R_t(i,j) \geq 0$ for all $i \neq j$, and
(ii) $\sum_{j \in \calS} R_t(i,j) = 0$ for all $i \in \calS$.
The off-diagonal entry $R_t(i,j)$ for $i \neq j$ represents the \emph{transition rate} from state $i$ to state $j$ at time $t$.
\end{definition}

\begin{definition}
For state $i \in \calS$ at time $t$:
\begin{equation}
\lambda_t(i) := \sum_{j \neq i} R_t(i,j) \quad \text{(exit rate)}, \qquad r_t(j \given i) := \frac{R_t(i,j)}{\lambda_t(i)} \quad \text{for } j \neq i \quad \text{(jump distribution)}.
\end{equation}
The exit rate $\lambda_t(i)$ controls the holding time in state $i$ (exponentially distributed), and $r_t(\cdot \given i)$ is a probability distribution over the $S - 1$ possible jump destinations.
\end{definition}

\begin{theorem}
\label{thm:kolmogorov}
Let $q_t(i) := \Pmeas(X_t = i)$ be the marginal distribution at time $t$. Then $q_t$ satisfies the master equation:
\begin{equation}
\frac{\dd q_t(i)}{\dd t} = \sum_{j \in \calS} q_t(j) R_t(j,i) = \sum_{j \neq i} q_t(j) R_t(j,i) - q_t(i) \lambda_t(i).
\end{equation}
\end{theorem}

\begin{definition}
Given a CTMC with rate matrix $R_t$ and marginal distribution $q_t(i)$ at time $t$, the \emph{time-reversed} CTMC has rate matrix $\Rhatmat_t$ given by:
\begin{equation}
\Rhatmat_t(i,j) = R_t(j,i) \frac{q_t(j)}{q_t(i)}, \quad i \neq j.
\end{equation}
\end{definition}

\subsection{Point Process}

\begin{definition}
A \emph{point process} on $[0,T]$ is a random collection of points (jump times) $\{t_1, t_2, \ldots\}$ in $[0,T]$. For a CTMC, the jump times form a point process whose intensity at time $t$ is $\lambda_t(X_t)$.
\end{definition}

\begin{theorem}
\label{thm:campbell_mecke}
Let $\Qmeas_{x_0}$ be the path measure of a CTMC starting from $x_0$ with rate matrix $R_t$. For any measurable function $f: [0,T] \times \calS \times \calS \to \R$, the expected sum over all jumps satisfies:
\begin{equation}
\E_{\Qmeas_{x_0}}\left[\sum_{k=1}^N f(t_k, X_{t_k^-}, X_{t_k})\right] = \int_0^T \sum_{i \in \calS} q_{t|0}(i \given x_0) \sum_{j \neq i} R_t(i,j) f(t,i,j) \, \dd t,
\end{equation}
where $N$ is the (random) number of jumps, $X_{t_k^-}$ is the state immediately before jump $k$, $X_{t_k}$ is the state after jump $k$, and $q_{t|0}(i \given x_0) := \Pmeas_{\Qmeas_{x_0}}(X_t = i)$ is the marginal distribution.
\end{theorem}

\section{Formal Problem Setup}
\label{app:problem_setup}

We establish here the formal mathematical framework for CTMC-based discrete diffusion models, providing explicit definitions for the objects introduced informally in the main text.

\subsection{Forward Process}

\begin{definition}
The \emph{forward} or \emph{noising} process is a time-dependent CTMC on $\calS$ with:
\begin{itemize}
    \item \textbf{Initial distribution}: $X_0 \sim \pdata$, the data distribution.
    \item \textbf{Rate matrix}: $R_t: \calS \times \calS \to \R$.
    \item \textbf{Transition rate}: $R_t(i,j) \geq 0$ for $i \neq j$.
    \item \textbf{Exit rate}: $\lambdafwd_t(i) := \sum_{j \neq i} R_t(i,j)$.
    \item \textbf{Jump distribution}: $r_t(j \given i) := \frac{R_t(i,j)}{\lambdafwd_t(i)}$ for $j \neq i$.
\end{itemize}
\end{definition}

\begin{definition}
For the forward process:
\begin{equation}
\begin{aligned}
q_{t|0}(x \given x_0) &:= \Pmeas(X_t = x \given X_0 = x_0), \quad \text{(conditional distribution)} \\
q_t(x) &:= \E_{x_0 \sim \pdata}[q_{t|0}(x \given x_0)], \quad \text{(marginal distribution)}.
\end{aligned}
\end{equation}
\end{definition}

\subsection{True Reverse Process}

\begin{definition}
The \emph{true reverse} process is the time-reversal of the forward process, characterized by the rate matrix:
\begin{equation}
\Rhatmat_t(i,j) = R_t(j,i) \frac{q_t(j)}{q_t(i)}, \quad i \neq j.
\end{equation}
This is a CTMC evolving from time $0$ to $T$ with:
\begin{equation}
\begin{aligned}
\lambdahat_t(i) &:= \sum_{j \neq i} \Rhatmat_t(i,j) \quad \text{(true reverse exit rate)}, \\
\hat{r}_t(j \given i) &:= \frac{\Rhatmat_t(i,j)}{\lambdahat_t(i)} \quad \text{(true reverse jump distribution)}.
\end{aligned}
\end{equation}
\end{definition}

\subsection{Neural Reverse Process}

\begin{definition}
The \emph{neural reverse process} or \emph{generative model} is a parameterized CTMC with:
\begin{itemize}
    \item \textbf{Initial distribution}: $X_T \sim \pref$ (reference distribution).
    \item \textbf{Rate matrix}: $\Rthetamat_t(i,j)$ parameterized by $\theta \in \Theta$.
    \item \textbf{Neural exit rate}: $\lambdatheta_t(i) := \sum_{j \neq i} \Rthetamat_t(i,j)$.
    \item \textbf{Neural jump distribution}: $r^\theta_t(j \given i) := \frac{\Rthetamat_t(i,j)}{\lambdatheta_t(i)}$.
\end{itemize}
\end{definition}

\begin{definition}
Typically, a neural network $\Phi_\theta: \calS \times [0,T] \to \R_{>0} \times \Delta^{S-1}$ outputs:
\begin{equation}
\Phi_\theta(x_t, t) = (\lambdatheta_t(x_t), r^\theta_t(\cdot \given x_t)),
\end{equation}
where $\Delta^{S-1}$ is the $(S-1)$-dimensional probability simplex (excluding self-transitions).
\end{definition}

\subsection{Training Objective}

\begin{definition}
The goal is to maximize the expected log-likelihood over the data distribution:
\begin{equation}
\max_{\theta \in \Theta} \E_{x_0 \sim \pdata}[\log p_\theta(x_0)].
\end{equation}
Equivalently, minimize the negative log-likelihood:
\begin{equation}
\min_{\theta \in \Theta} \E_{x_0 \sim \pdata}[-\log p_\theta(x_0)].
\end{equation}
\end{definition}

\begin{remark}[Intractability]
Direct evaluation of $p_\theta(x_0)$ requires marginalizing over all possible paths in continuous time, which is computationally intractable. This motivates the ELBO approach developed in Appendix~\ref{app:proofs}.
\end{remark}

\section{Proofs}
\label{app:proofs}

\label{app:proof_path_measure}

\begin{lemma}
Let $\Qmeas_{x_0}$ denote the path measure induced by a CTMC with rate matrix $R_t$ starting from initial state $x_0$. For a path $\omega$ with jumps at times $0 \leq t_1 < \cdots < t_n < T$ and states $x_0 \to x_1 \to \cdots \to x_n$:
\begin{equation}
\dd\Qmeas_{x_0}(\omega) = \left[\prod_{k=1}^n R_{t_k}(x_{k-1}, x_k)\right] \exp\left(-\int_0^T \lambda_t(X_t) \, \dd t\right).
\end{equation}
\end{lemma}

\begin{proof}
We decompose the path probability into a product of conditional distributions following the CTMC dynamics. A trajectory starting from $x_0$ is fully described by jump times $0 \leq t_1 < \cdots < t_n < T$ and states $x_0 \to x_1 \to \cdots \to x_n$, where at jump time $t_k$ the chain transitions from $x_{k-1}$ to $x_k$, and $X_t = x_k$ for $t \in [t_k, t_{k+1})$ (with $t_0 := 0$ and $t_{n+1} := T$).

\textbf{Step 1: Holding time distribution.}
When the CTMC is in state $x_{k-1}$ starting from time $t_{k-1}$, the waiting time until the next jump is governed by a time-inhomogeneous exponential distribution. The cumulative distribution function for the next jump time $t_k$ is:
\begin{equation}
\Pmeas(T_k < t \given T_{k-1} = t_{k-1}) = 1 - \exp\left(-\int_{t_{k-1}}^{t} \lambda_s(x_{k-1}) \, \dd s\right).
\end{equation}
Differentiating, the density of the $k$-th jump time is:
\begin{equation}
p(t_k \given T_{k-1} = t_{k-1}) = \lambda_{t_k}(x_{k-1}) \exp\left(-\int_{t_{k-1}}^{t_k} \lambda_s(x_{k-1}) \, \dd s\right).
\end{equation}

\textbf{Step 2: Jump destination distribution.}
Conditional on a jump occurring at time $t_k$ from state $x_{k-1}$, the destination $x_k$ is chosen according to:
\begin{equation}
\Pmeas(X_{t_k} = x_k \given X_{t_k^-} = x_{k-1}) = r_{t_k}(x_k \given x_{k-1}) = \frac{R_{t_k}(x_{k-1}, x_k)}{\lambda_{t_k}(x_{k-1})}.
\end{equation}

\textbf{Step 3: Combining the contributions.}
The joint contribution of the $k$-th jump (both timing and destination) is:
\begin{equation}
\begin{aligned}
&p(t_k \given T_{k-1} = t_{k-1}) \cdot \Pmeas(X_{t_k} = x_k \given X_{t_k^-} = x_{k-1}) \\
&\quad = \lambda_{t_k}(x_{k-1}) \exp\left(-\int_{t_{k-1}}^{t_k} \lambda_s(x_{k-1}) \, \dd s\right) \cdot \frac{R_{t_k}(x_{k-1}, x_k)}{\lambda_{t_k}(x_{k-1})} \\
&\quad = R_{t_k}(x_{k-1}, x_k) \exp\left(-\int_{t_{k-1}}^{t_k} \lambda_s(x_{k-1}) \, \dd s\right).
\end{aligned}
\end{equation}
The survival probability after the last jump (no further jump in $[t_n, T]$) is:
\begin{equation}
\Pmeas(T_{n+1} \geq T \given T_n = t_n) = \exp\left(-\int_{t_n}^{T} \lambda_s(x_n) \, \dd s\right).
\end{equation}

\textbf{Step 4: Full path density.}
Taking the product over all $n$ jumps and the final survival term:
\begin{equation}
\dd\Qmeas_{x_0}(\omega) = \prod_{k=1}^n R_{t_k}(x_{k-1}, x_k) \cdot \prod_{k=1}^n \exp\left(-\int_{t_{k-1}}^{t_k} \lambda_s(x_{k-1}) \, \dd s\right) \cdot \exp\left(-\int_{t_n}^{T} \lambda_s(x_n) \, \dd s\right).
\end{equation}
Since $X_t = x_{k-1}$ for $t \in [t_{k-1}, t_k)$, the exponential terms combine into a single integral over the entire time horizon:
\begin{equation}
\prod_{k=1}^n \exp\left(-\int_{t_{k-1}}^{t_k} \lambda_s(x_{k-1}) \, \dd s\right) \cdot \exp\left(-\int_{t_n}^{T} \lambda_s(x_n) \, \dd s\right) = \exp\left(-\int_0^T \lambda_t(X_t) \, \dd t\right),
\end{equation}
which yields the desired result.
\end{proof}

\begin{lemma}
Let $\Qmeas_{x_0}$ be the forward path measure starting from $x_0$, and $\Pmeas^{\theta}$ be the neural reverse path measure. For a path $\omega$ with jumps at times $0 < t_1 < \cdots < t_n < T$ and visited states $x_0 \to x_1 \to \cdots \to x_n$, the log Radon-Nikodym derivative is:
\begin{equation}
\label{eq:log_ratio_app}
\begin{split}
\log \frac{\dd\Pmeas^{\theta}}{\dd\Qmeas_{x_0}}(\omega) &= \log \pref(x_n) + \sum_{k=1}^n \log \frac{\Rthetamat_{t_k}(x_k, x_{k-1})}{R_{t_k}(x_{k-1}, x_k)} \\
&\quad + \int_0^T \left[\lambdafwd_t(X_t(\omega)) - \lambdatheta_t(X_t(\omega))\right] \dd t.
\end{split}
\end{equation}
\end{lemma}

\begin{proof}
From Lemma~\ref{lem:path_measure}, the forward path density (conditional on $X_0 = x_0$) is:
\begin{equation}
\dd\Qmeas_{x_0}(\omega) = \left[\prod_{k=1}^n R_{t_k}(x_{k-1}, x_k)\right] \exp\left(-\int_0^T \lambdafwd_t(X_t) \, \dd t\right).
\end{equation}

The neural reverse path density for the \emph{same path} (note the index reversal: jumps are $x_k \to x_{k-1}$ in reverse time) is:
\begin{equation}
\dd\Pmeas^{\theta}(\omega) = \pref(x_n) \left[\prod_{k=1}^n \Rthetamat_{t_k}(x_k, x_{k-1})\right] \exp\left(-\int_0^T \lambdatheta_t(X_t) \, \dd t\right).
\end{equation}

Taking the ratio:
\begin{equation}
\frac{\dd\Pmeas^{\theta}}{\dd\Qmeas_{x_0}}(\omega) = \pref(x_n) \cdot \frac{\prod_{k=1}^n \Rthetamat_{t_k}(x_k, x_{k-1})}{\prod_{k=1}^n R_{t_k}(x_{k-1}, x_k)} \cdot \exp\left(\int_0^T \left[\lambdafwd_t(X_t) - \lambdatheta_t(X_t)\right] \dd t\right).
\end{equation}

Taking logarithms yields~\eqref{eq:log_ratio_app}.
\end{proof}

\begin{theorem}
\label{thm:single_elbo_app}
For any data point $x_0 \in \calS$, the negative log-likelihood satisfies:
\begin{equation}
\label{eq:single_elbo_app}
-\log p_\theta(x_0) \leq -\E_{q_{T|0}}[\log \pref(X_T)] + \int_0^T \sum_{i \in \calS} q_{t|0}(i \given x_0) \, \ell_t(i) \, \dd t,
\end{equation}
where the integrand $\ell_t(i)$ is defined as:
\begin{equation}
\label{eq:ell_t_app}
\ell_t(i) := \sum_{j \neq i} \left[R_t(i,j) \log \frac{R_t(i,j)}{\Rthetamat_t(j,i)} - R_t(i,j) + \Rthetamat_t(i,j)\right].
\end{equation}
\end{theorem}

\begin{proof}
\textbf{Step 1: Importance sampling.}

Since the forward path measure $\Qmeas_{x_0}$ has $X_0 = x_0$ with probability $1$, we can write:
\begin{equation}
\begin{aligned}
p_\theta(x_0) &= \int_\Omega \indicator[X_0(\omega) = x_0] \, \dd\Pmeas^{\theta}(\omega) \\
&= \int_\Omega \indicator[X_0(\omega) = x_0] \cdot \frac{\dd\Pmeas^{\theta}}{\dd\Qmeas_{x_0}}(\omega) \, \dd\Qmeas_{x_0}(\omega) \quad \text{(Radon-Nikodym)} \\
&= \int_\Omega \frac{\dd\Pmeas^{\theta}}{\dd\Qmeas_{x_0}}(\omega) \, \dd\Qmeas_{x_0}(\omega) \quad \text{(since $\Qmeas_{x_0}$ starts from $x_0$)} \\
&= \E_{\Qmeas_{x_0}}\left[\frac{\dd\Pmeas^{\theta}}{\dd\Qmeas_{x_0}}\right].
\end{aligned}
\end{equation}

\textbf{Step 2: Jensen's inequality.}

Since $\log$ is concave, we have:
\begin{equation}
\log p_\theta(x_0) = \log \E_{\Qmeas_{x_0}}\left[\frac{\dd\Pmeas^{\theta}}{\dd\Qmeas_{x_0}}\right] \geq \E_{\Qmeas_{x_0}}\left[\log \frac{\dd\Pmeas^{\theta}}{\dd\Qmeas_{x_0}}\right].
\end{equation}

Negating both sides:
\begin{equation}
\label{eq:jensen_neg_app}
-\log p_\theta(x_0) \leq -\E_{\Qmeas_{x_0}}\left[\log \frac{\dd\Pmeas^{\theta}}{\dd\Qmeas_{x_0}}\right].
\end{equation}

\textbf{Step 3: Expand using Lemma~3.1.}

Substituting~\eqref{eq:log_ratio_app} into~\eqref{eq:jensen_neg_app}:
\begin{equation}
\begin{aligned}
-\log p_\theta(x_0) &\leq -\E_{\Qmeas_{x_0}}[\log \pref(X_T)] \\
&\quad - \E_{\Qmeas_{x_0}}\left[\sum_{k=1}^N \log \frac{\Rthetamat_{t_k}(X_{t_k}, X_{t_k^-})}{R_{t_k}(X_{t_k^-}, X_{t_k})}\right] \\
&\quad - \E_{\Qmeas_{x_0}}\left[\int_0^T (\lambdafwd_t(X_t) - \lambdatheta_t(X_t)) \, \dd t\right],
\end{aligned}
\end{equation}
where $N$ is the random number of jumps.

\textbf{Step 4: Apply Campbell-Mecke formula (Theorem~\ref{thm:campbell_mecke}).}

For the jump term, set $f(t,i,j) = \log \frac{R_t(i,j)}{\Rthetamat_t(j,i)}$:
\begin{equation}
\E_{\Qmeas_{x_0}}\left[\sum_{k=1}^N \log \frac{R_{t_k}(X_{t_k^-}, X_{t_k})}{\Rthetamat_{t_k}(X_{t_k}, X_{t_k^-})}\right] = \int_0^T \sum_{i \in \calS} q_{t|0}(i \given x_0) \sum_{j \neq i} R_t(i,j) \log \frac{R_t(i,j)}{\Rthetamat_t(j,i)} \, \dd t.
\end{equation}

For the survival term, by Fubini's theorem:
\begin{equation}
\begin{aligned}
\E_{\Qmeas_{x_0}}\left[\int_0^T (\lambdafwd_t(X_t) - \lambdatheta_t(X_t)) \, \dd t\right] &= \int_0^T \E_{\Qmeas_{x_0}}[\lambdafwd_t(X_t) - \lambdatheta_t(X_t)] \, \dd t \\
&= \int_0^T \sum_{i \in \calS} q_{t|0}(i \given x_0) \left[\lambdafwd_t(i) - \lambdatheta_t(i)\right] \dd t \\
&= \int_0^T \sum_{i \in \calS} q_{t|0}(i \given x_0) \sum_{j \neq i} \left[R_t(i,j) - \Rthetamat_t(i,j)\right] \dd t.
\end{aligned}
\end{equation}

\textbf{Step 5: Combine jump and survival terms.}

From Steps 3 and 4, we need to compute:
\begin{equation}
\label{eq:combine_app}
-\E_{\Qmeas_{x_0}}\left[\log \frac{\dd\Pmeas^{\theta}}{\dd\Qmeas_{x_0}}\right] = -\E_{\Qmeas_{x_0}}[\log \pref(X_T)] + I_{\text{jump}} - I_{\text{surv}},
\end{equation}
where the jump and survival integrals are (from Step 4):
\begin{align}
I_{\text{jump}} &= \int_0^T \sum_{i \in \calS} q_{t|0}(i \given x_0) \sum_{j \neq i} R_t(i,j) \log \frac{R_t(i,j)}{\Rthetamat_t(j,i)} \, \dd t, \\
I_{\text{surv}} &= \int_0^T \sum_{i \in \calS} q_{t|0}(i \given x_0) \sum_{j \neq i} \left[R_t(i,j) - \Rthetamat_t(i,j)\right] \dd t.
\end{align}

Combining~(39) and the survival term:
\begin{equation}
\begin{aligned}
I_{\text{jump}} - I_{\text{surv}} &= \int_0^T \sum_{i \in \calS} q_{t|0}(i \given x_0) \sum_{j \neq i} \left[R_t(i,j) \log \frac{R_t(i,j)}{\Rthetamat_t(j,i)} - \left(R_t(i,j) - \Rthetamat_t(i,j)\right)\right] \dd t \\
&= \int_0^T \sum_{i \in \calS} q_{t|0}(i \given x_0) \sum_{j \neq i} \left[R_t(i,j) \log \frac{R_t(i,j)}{\Rthetamat_t(j,i)} - R_t(i,j) + \Rthetamat_t(i,j)\right] \dd t \\
&= \int_0^T \sum_{i \in \calS} q_{t|0}(i \given x_0) \, \ell_t(i) \, \dd t,
\end{aligned}
\end{equation}
where we define:
\begin{equation}
\ell_t(i) := \sum_{j \neq i} \left[R_t(i,j) \log \frac{R_t(i,j)}{\Rthetamat_t(j,i)} - R_t(i,j) + \Rthetamat_t(i,j)\right],
\end{equation}
which exactly matches~\eqref{eq:ell_t_app}. This completes the proof.
\end{proof}

\begin{theorem}
\label{thm:data_elbo_app}
The expected negative log-likelihood over the data distribution satisfies:
\begin{equation}
\label{eq:data_elbo_app}
\E_{x_0 \sim \pdata}[-\log p_\theta(x_0)] \leq \KLdiv(\hat{\Qmeas} \| \Pmeas^{\theta}) + C_{\text{total}},
\end{equation}
where $\hat{\Qmeas}$ is the path measure of the true reverse process, $\Pmeas^{\theta}$ is the path measure of the neural reverse process, and $C_{\text{total}}$ is a constant independent of $\theta$. Explicitly:
\begin{equation}
\label{eq:kl_explicit_app}
\KLdiv(\hat{\Qmeas} \| \Pmeas^{\theta}) = \int_0^T \sum_{i \in \calS} q_t(i) \sum_{j \neq i} \left[\Rhatmat_t(i,j) \log \frac{\Rhatmat_t(i,j)}{\Rthetamat_t(i,j)} - \Rhatmat_t(i,j) + \Rthetamat_t(i,j)\right] \dd t.
\end{equation}
\end{theorem}

\begin{proof}
\textbf{Step 1: Average over data.}

Taking expectation of both sides of~\eqref{eq:single_elbo_app} over $x_0 \sim \pdata$:
\begin{equation}
\E_{\pdata}[-\log p_\theta(x_0)] \leq \E_{\pdata}[-\E_{q_{T|0}}[\log \pref(X_T)]] + \E_{\pdata}\left[\int_0^T \sum_{i \in \calS} q_{t|0}(i \given x_0) \, \ell_t(i) \, \dd t\right].
\end{equation}

\textbf{Step 2: Define marginal distribution.}

Define the (unconditional) marginal distribution at time $t$:
\begin{equation}
q_t(i) := \E_{x_0 \sim \pdata}[q_{t|0}(i \given x_0)] = \int_{\calS} \pdata(x_0) \, q_{t|0}(i \given x_0).
\end{equation}

By Fubini's theorem:
\begin{equation}
\E_{\pdata}\left[\int_0^T \sum_{i \in \calS} q_{t|0}(i \given x_0) \, \ell_t(i) \, \dd t\right] = \int_0^T \sum_{i \in \calS} q_t(i) \, \ell_t(i) \, \dd t.
\end{equation}

\textbf{Step 3: Transform jump term.}

Expand $\ell_t(i)$ from~\eqref{eq:ell_t_app}:
\begin{equation}
\label{eq:expand_ell_app}
q_t(i) \, \ell_t(i) = \sum_{j \neq i} q_t(i) \left[R_t(i,j) \log \frac{R_t(i,j)}{\Rthetamat_t(j,i)} - R_t(i,j) + \Rthetamat_t(i,j)\right].
\end{equation}

For the \textbf{jump term}, since $\Rhatmat_t(j,i) = R_t(i,j) \frac{q_t(i)}{q_t(j)}$, we have $q_t(i) R_t(i,j) = q_t(j) \Rhatmat_t(j,i)$. Substitute:
\begin{equation}
q_t(i) R_t(i,j) \log \frac{R_t(i,j)}{\Rthetamat_t(j,i)} = q_t(j) \Rhatmat_t(j,i) \log \frac{R_t(i,j)}{\Rthetamat_t(j,i)}.
\end{equation}

Since $R_t(i,j) = \Rhatmat_t(j,i) \frac{q_t(j)}{q_t(i)}$:
\begin{equation}
\begin{aligned}
q_t(j) \Rhatmat_t(j,i) \log \frac{R_t(i,j)}{\Rthetamat_t(j,i)} &= q_t(j) \Rhatmat_t(j,i) \log \frac{\Rhatmat_t(j,i) \, q_t(j) / q_t(i)}{\Rthetamat_t(j,i)} \\
&= q_t(j) \Rhatmat_t(j,i) \left[\log \frac{\Rhatmat_t(j,i)}{\Rthetamat_t(j,i)} + \log \frac{q_t(j)}{q_t(i)}\right].
\end{aligned}
\end{equation}

Summing over $i$ and $j$ (relabel $i \leftrightarrow j$):
\begin{equation}
\begin{aligned}
\sum_{i \in \calS} \sum_{j \neq i} q_t(i) R_t(i,j) \log \frac{R_t(i,j)}{\Rthetamat_t(j,i)} &= \sum_{i \in \calS} \sum_{j \neq i} q_t(i) \Rhatmat_t(i,j) \log \frac{\Rhatmat_t(i,j)}{\Rthetamat_t(i,j)} \\
&\quad + \underbrace{\sum_{i \in \calS} \sum_{j \neq i} q_t(i) \Rhatmat_t(i,j) \log \frac{q_t(i)}{q_t(j)}}_{C_t^{(1)} \text{ (independent of } \theta)}.
\end{aligned}
\end{equation}

\textbf{Step 4: Transform survival term.}

For the \textbf{survival term} $\sum_{j \neq i} [-R_t(i,j) + \Rthetamat_t(i,j)]$, since $q_t(i) R_t(i,j) = q_t(j) \Rhatmat_t(j,i)$:
\begin{equation}
\sum_{i \in \calS} q_t(i) \lambdafwd_t(i) = \sum_{i \in \calS} q_t(i) \sum_{j \neq i} R_t(i,j) = \sum_{i \in \calS} \sum_{j \neq i} q_t(j) \Rhatmat_t(j,i) = \sum_{i \in \calS} q_t(i) \lambdahat_t(i).
\end{equation}

Therefore:
\begin{equation}
\begin{aligned}
\sum_{i \in \calS} q_t(i) \sum_{j \neq i} [-R_t(i,j) + \Rthetamat_t(i,j)] &= -\sum_{i \in \calS} q_t(i) \lambdafwd_t(i) + \sum_{i \in \calS} q_t(i) \lambdatheta_t(i) \\
&= -\sum_{i \in \calS} q_t(i) \lambdahat_t(i) + \sum_{i \in \calS} q_t(i) \lambdatheta_t(i) \\
&= \sum_{i \in \calS} q_t(i) \sum_{j \neq i} [\Rthetamat_t(i,j) - \Rhatmat_t(i,j)].
\end{aligned}
\end{equation}

\textbf{Step 5: Combine jump and survival terms.}

Putting everything together from Steps 3 and 4:
\begin{equation}
\begin{aligned}
\int_0^T \sum_{i \in \calS} q_t(i) \, \ell_t(i) \, \dd t &= \int_0^T \sum_{i \in \calS} q_t(i) \sum_{j \neq i} \left[\Rhatmat_t(i,j) \log \frac{\Rhatmat_t(i,j)}{\Rthetamat_t(i,j)} + \Rthetamat_t(i,j) - \Rhatmat_t(i,j)\right] \dd t \\
&\quad + \int_0^T C_t^{(1)} \, \dd t.
\end{aligned}
\end{equation}

The first term can be rewritten as:
\begin{equation}
\int_0^T \sum_{i \in \calS} q_t(i) \sum_{j \neq i} \left[\Rhatmat_t(i,j) \log \frac{\Rhatmat_t(i,j)}{\Rthetamat_t(i,j)} - \Rhatmat_t(i,j) + \Rthetamat_t(i,j)\right] \dd t = \KLdiv(\hat{\Qmeas} \| \Pmeas^{\theta}),
\end{equation}
and the second term (plus the prior term) is $C_{\text{total}}$, independent of $\theta$.
\end{proof}

\begin{corollary}
\label{cor:decomp_app}
Define the true reverse quantities:
\begin{equation}
\begin{aligned}
\lambdahat_t(i) &:= \sum_{j \neq i} \Rhatmat_t(i,j) \quad \text{(exit rate)}, \\
\hat{r}_t(j \given i) &:= \frac{\Rhatmat_t(i,j)}{\lambdahat_t(i)} \quad \text{for } j \neq i \quad \text{(jump distribution)},
\end{aligned}
\end{equation}
and similarly for the neural reverse:
\begin{equation}
\begin{aligned}
\lambdatheta_t(i) &:= \sum_{j \neq i} \Rthetamat_t(i,j), \\
r^\theta_t(j \given i) &:= \frac{\Rthetamat_t(i,j)}{\lambdatheta_t(i)}.
\end{aligned}
\end{equation}
Then:
\begin{equation}
\label{eq:decomp_app}
\E_{\pdata}[-\log p_\theta(x_0)] \leq \int_0^T \sum_{i \in \calS} q_t(i) \left[\KLdiv^{\Poi}(\lambdahat_t(i) \| \lambdatheta_t(i)) + \lambdahat_t(i) \cdot \KLdiv^{\Cat}(\hat{r}_t(\cdot \given i) \| r^\theta_t(\cdot \given i))\right] \dd t + C,
\end{equation}
where:
\begin{equation}
\begin{aligned}
\KLdiv^{\Poi}(\lambda \| \lambda^{\theta}) &:= \lambda \log \frac{\lambda}{\lambda^{\theta}} - \lambda + \lambda^{\theta}, \\
\KLdiv^{\Cat}(\hat{r} \| r^{\theta}) &:= \sum_{j \neq i} \hat{r}(j \given i) \log \frac{\hat{r}(j \given i)}{r^{\theta}(j \given i)}.
\end{aligned}
\end{equation}
\end{corollary}

\begin{proof}
Substitute $\Rhatmat_t(i,j) = \lambdahat_t(i) \hat{r}_t(j \given i)$ and $\Rthetamat_t(i,j) = \lambdatheta_t(i) r^\theta_t(j \given i)$ into~\eqref{eq:kl_explicit_app}:
\begin{equation}
\begin{aligned}
&\sum_{j \neq i} \left[\Rhatmat_t(i,j) \log \frac{\Rhatmat_t(i,j)}{\Rthetamat_t(i,j)} - \Rhatmat_t(i,j) + \Rthetamat_t(i,j)\right] \\
&= \sum_{j \neq i} \lambdahat_t(i) \hat{r}_t(j \given i) \log \frac{\lambdahat_t(i) \hat{r}_t(j \given i)}{\lambdatheta_t(i) r^\theta_t(j \given i)} - \lambdahat_t(i) + \lambdatheta_t(i) \\
&= \lambdahat_t(i) \log \frac{\lambdahat_t(i)}{\lambdatheta_t(i)} + \lambdahat_t(i) \sum_{j \neq i} \hat{r}_t(j \given i) \log \frac{\hat{r}_t(j \given i)}{r^\theta_t(j \given i)} - \lambdahat_t(i) + \lambdatheta_t(i) \\
&= \KLdiv^{\Poi}(\lambdahat_t(i) \| \lambdatheta_t(i)) + \lambdahat_t(i) \cdot \KLdiv^{\Cat}(\hat{r}_t(\cdot \given i) \| r^\theta_t(\cdot \given i)).
\end{aligned}
\end{equation}
\end{proof}

\begin{corollary}
\label{cor:mc_loss_app}
The training objective~\eqref{eq:decomp} can be estimated via Monte Carlo sampling:
\begin{equation}
\label{eq:mc_loss_app}
\mathcal{L}(\theta) := \E_{t \sim \mathcal{U}(0,T), \, x_t \sim q_t}\left[\KLdiv^{\Poi}(\lambdahat_t(x_t) \| \lambdatheta_t(x_t)) + \lambdahat_t(x_t) \cdot \E_{j \sim \hat{r}_t(\cdot \given x_t)}\left[-\log r^\theta_t(j \given x_t)\right] - \lambdahat_t(x_t) H(\hat{r}_t)\right],
\end{equation}
where $H(\hat{r}_t) := -\sum_{j \neq i} \hat{r}_t(j \given i) \log \hat{r}_t(j \given i)$ is the entropy of the true reverse jump distribution. Equivalently, using the cross-entropy (CE) identity $\KLdiv^{\Cat} = CE - H$:
\begin{equation}
\label{eq:mc_loss_ce_app}
\mathcal{L}(\theta) = \E_{t,x_t}\left[\KLdiv^{\Poi}(\lambdahat_t(x_t) \| \lambdatheta_t(x_t)) + \lambdahat_t(x_t) \cdot CE(\hat{r}_t(\cdot \given x_t), r^\theta_t(\cdot \given x_t))\right] + \text{const}.
\end{equation}
\end{corollary}

\begin{proof}
Direct application of Monte Carlo integration to~\eqref{eq:decomp}:
\begin{equation}
\int_0^T \sum_{i \in \calS} q_t(i) \cdot (\cdots) \, \dd t = T \cdot \E_{t \sim \mathcal{U}(0,T)}\left[\sum_{i \in \calS} q_t(i) \cdot (\cdots)\right] = T \cdot \E_{t \sim \mathcal{U}(0,T), \, x_t \sim q_t}[(\cdots)].
\end{equation}

The categorical KL can be written as:
\begin{equation}
\KLdiv^{\Cat}(\hat{r}_t \| r^\theta_t) = \E_{j \sim \hat{r}_t}[-\log r^\theta_t(j \given x_t)] - H(\hat{r}_t),
\end{equation}
where the first term is the cross-entropy and the second is the entropy (independent of $\theta$).
\end{proof}

\begin{lemma}
\label{lem:expectation_transform}
Let $p_0$ be a distribution over $\calS$, and let $p(\cdot \given x_0)$ be conditional distributions such that $p(x) = \sum_{x_0 \in \calS} p_0(x_0) p(x \given x_0)$ is the marginal distribution. For any measurable function $f: \calS \times \calS \to \R$, we have:
\begin{equation}
\label{eq:expectation_transform_app}
\E_{x \sim p}\left[\sum_{y \neq x} f(x,y) \frac{p(y)}{p(x)}\right] = \E_{x_0 \sim p_0, \, x \sim p(\cdot \given x_0)}\left[\sum_{y \neq x} f(x,y) \frac{p(y \given x_0)}{p(x \given x_0)}\right].
\end{equation}
\end{lemma}

\begin{proof}
\begin{equation}
\begin{aligned}
\E_{x \sim p}\left[\sum_{y \neq x} f(x,y) \frac{p(y)}{p(x)}\right] &= \sum_{x \in \calS} p(x) \sum_{y \neq x} f(x,y) \frac{p(y)}{p(x)} \\
&= \sum_{x \in \calS} \sum_{y \neq x} f(x,y) p(y) \\
&= \sum_{x \in \calS} \sum_{y \neq x} f(x,y) \sum_{x_0 \in \calS} p_0(x_0) p(y \given x_0) \\
&= \sum_{x_0 \in \calS} p_0(x_0) \sum_{x \in \calS} \sum_{y \neq x} f(x,y) p(y \given x_0) \\
&= \sum_{x_0 \in \calS} p_0(x_0) \sum_{x \in \calS} p(x \given x_0) \sum_{y \neq x} f(x,y) \frac{p(y \given x_0)}{p(x \given x_0)} \\
&= \E_{x_0 \sim p_0, \, x \sim p(\cdot \given x_0)}\left[\sum_{y \neq x} f(x,y) \frac{p(y \given x_0)}{p(x \given x_0)}\right].
\end{aligned}
\end{equation}
\end{proof}

\begin{proposition}
\label{prop:conditional_reverse}
For a fixed data point $x_0 \in \calS$, define the conditional reverse rates:
\begin{equation}
\label{eq:conditional_reverse_matrix_app}
\hat{R}_{t|0}(j, i \given x_0) := R_t(i, j) \frac{q_{t|0}(i \given x_0)}{q_{t|0}(j \given x_0)}, \quad i \neq j,
\end{equation}
\begin{equation}
\label{eq:conditional_exit_rate_app}
\hat{\lambda}_{t|0}(i \given x_0) := \sum_{j \neq i} \hat{R}_{t|0}(i, j \given x_0) = \sum_{j \neq i} R_t(j, i) \frac{q_{t|0}(j \given x_0)}{q_{t|0}(i \given x_0)},
\end{equation}
\begin{equation}
\label{eq:conditional_jump_dist_app}
\hat{r}_{t|0}(j \given i, x_0) := \frac{\hat{R}_{t|0}(i, j \given x_0)}{\hat{\lambda}_{t|0}(i \given x_0)} = \frac{R_t(j, i) \frac{q_{t|0}(j \given x_0)}{q_{t|0}(i \given x_0)}}{\hat{\lambda}_{t|0}(i \given x_0)}, \quad j \neq i.
\end{equation}
Then the training loss~\eqref{eq:mc_loss_ce_app} can be equivalently written as:
\begin{equation}
\label{eq:conditional_loss_app}
\mathcal{L}(\theta) = \E_{t \sim \mathcal{U}(0,T), \, x_0 \sim \pdata, \, i \sim q_{t|0}(\cdot \given x_0)}\left[- \sum_{j \neq i} R_t(j, i) \frac{q_{t|0}(j \given x_0)}{q_{t|0}(i \given x_0)} \log \lambdatheta_t(i) r^\theta_t(j \given i) - \hat{\lambda}_{t|0}(i \given x_0) + \lambdatheta_t(i)\right] + \text{const}.
\end{equation}
\end{proposition}

\begin{proof}
We expand the training loss~\eqref{eq:mc_loss_ce_app} and apply Lemma~\ref{lem:expectation_transform} to transform marginal distributions to conditional distributions.

Starting from~\eqref{eq:mc_loss_ce_app}, expand the KL divergence and cross-entropy:
\begin{equation}
\begin{aligned}
\mathcal{L}(\theta) &= \E_{t \sim \mathcal{U}(0,T), \, i \sim q_t}\left[\KLdiv^{\Poi}(\hat{\lambda}_t(i) \| \lambdatheta_t(i)) + \hat{\lambda}_t(i) \cdot CE(\hat{r}_t(\cdot \given i), r^\theta_t(\cdot \given i))\right] + \text{const} \\
&= \E_{t \sim \mathcal{U}(0,T), \, i \sim q_t}\left[\hat{\lambda}_t(i) \log \frac{\hat{\lambda}_t(i)}{\lambdatheta_t(i)} - \hat{\lambda}_t(i) + \lambdatheta_t(i) - \hat{\lambda}_t(i) \sum_{j \neq i} \hat{r}_t(j \given i) \log r^\theta_t(j \given i)\right] + \text{const}.
\end{aligned}
\end{equation}

Using $\hat{\lambda}_t(i) = \sum_{j \neq i} \hat{R}_t(i,j) = \sum_{j \neq i} R_t(j,i) \frac{q_t(j)}{q_t(i)}$ and $\hat{\lambda}_t(i) \hat{r}_t(j \given i) = \hat{R}_t(i,j) = R_t(j,i) \frac{q_t(j)}{q_t(i)}$, we can write:

The first term:
\begin{equation}
\begin{aligned}
\E_{t \sim \mathcal{U}(0,T), \, i \sim q_t}\left[\hat{\lambda}_t(i) \log \frac{\hat{\lambda}_t(i)}{\lambdatheta_t(i)}\right] &= \E_{t, i}\left[\sum_{j \neq i} R_t(j,i) \frac{q_t(j)}{q_t(i)} \log \frac{\sum_{k \neq i} R_t(k,i) \frac{q_t(k)}{q_t(i)}}{\lambdatheta_t(i)}\right] \\
&= \E_{t, i}\left[\sum_{j \neq i} R_t(j,i) \frac{q_t(j)}{q_t(i)} \log \frac{\sum_{k \neq i} R_t(k,i) \frac{q_t(k)}{q_t(i)}}{\lambdatheta_t(i)}\right] \\
&\quad - \E_{t \sim \mathcal{U}(0,T), \, x_0 \sim \pdata, \, i \sim q_{t|0}(\cdot \given x_0)}\left[\sum_{j \neq i} R_t(j,i) \frac{q_{t|0}(j \given x_0)}{q_{t|0}(i \given x_0)} \log \lambdatheta_t(i)\right] \\
&= \text{const} - \E_{t \sim \mathcal{U}(0,T), \, x_0 \sim \pdata, \, i \sim q_{t|0}(\cdot \given x_0)}\left[\hat{\lambda}_{t|0}(i \given x_0) \log \lambdatheta_t(i)\right].
\end{aligned}
\end{equation}

The second term:
\begin{equation}
\begin{aligned}
\E_{t \sim \mathcal{U}(0,T), \, i \sim q_t}\left[-\hat{\lambda}_t(i)\right] &= \E_{t, i}\left[-\sum_{j \neq i} R_t(j,i) \frac{q_t(j)}{q_t(i)}\right] \\
&= \E_{t \sim \mathcal{U}(0,T), \, x_0 \sim \pdata, \, i \sim q_{t|0}(\cdot \given x_0)}\left[-\sum_{j \neq i} R_t(j,i) \frac{q_{t|0}(j \given x_0)}{q_{t|0}(i \given x_0)}\right] \\
&= \E_{t \sim \mathcal{U}(0,T), \, x_0 \sim \pdata, \, i \sim q_{t|0}(\cdot \given x_0)}\left[-\hat{\lambda}_{t|0}(i \given x_0)\right].
\end{aligned}
\end{equation}

The third term:
\begin{equation}
\begin{aligned}
\E_{t \sim \mathcal{U}(0,T), \, i \sim q_t}\left[-\hat{\lambda}_t(i) \sum_{j \neq i} \hat{r}_t(j \given i) \log r^\theta_t(j \given i)\right] &= \E_{t, i}\left[-\sum_{j \neq i} \hat{\lambda}_t(i) \hat{r}_t(j \given i) \log r^\theta_t(j \given i)\right] \\
&= \E_{t, i}\left[-\sum_{j \neq i} R_t(j,i) \frac{q_t(j)}{q_t(i)} \log r^\theta_t(j \given i)\right] \\
&= \E_{t, \, x_0, \, i}\left[-\sum_{j \neq i} R_t(j,i) \frac{q_{t|0}(j \given x_0)}{q_{t|0}(i \given x_0)} \log r^\theta_t(j \given i)\right] \\
\end{aligned}
\end{equation}

Therefore, combining the three terms leads to the conclusion.

\end{proof}

\begin{proposition}
\label{prop:exact_gradient_app}
Assume: (i) forward quantities ($q_t$, $q_{t|0}$, $p(x_0\given x_t)$) are independent of $\theta$; (ii) $\Rthetamat_t(i,j) > 0$ and differentiable in $\theta$ for all $i\neq j$; and (iii) differentiation can be exchanged with expectation/integration. Then there exists a $\theta$-independent constant $C_{\mathrm{gap}} \ge 0$ such that
\[
\mathcal{L}_{\mathrm{KL}}(\theta) = \KLdiv(\hat{\Qmeas}\|\Pmeas^\theta) + C_{\mathrm{gap}},
\]
hence
\[
\nabla_\theta \mathcal{L}_{\mathrm{KL}}(\theta) = \nabla_\theta \KLdiv(\hat{\Qmeas}\|\Pmeas^\theta).
\]
In particular,
\[
\arg\min_\theta \mathcal{L}_{\mathrm{KL}}(\theta) = \arg\min_\theta \KLdiv(\hat{\Qmeas}\|\Pmeas^\theta),
\]
and every stationary point of one objective is a stationary point of the other.
\end{proposition}

\begin{proof}
\label{app:proof_exact_gradient}
For fixed $(t,i,j)$, let $R := \Rhatmat_t(i,j\given X_0)$ with $X_0 \sim p(x_0\given x_t=i)$ and let $c := \Rthetamat_t(i,j)$. Since $\E[R] = \Rhatmat_t(i,j)$, expanding $f(r,c)=r\log \frac{r}{c}-r+c$ gives
\[
\E[f(R,c)] - f(\E[R],c) = \E[R\log R] - \E[R]\log \E[R] =: \Delta_{t,i,j}\ge 0,
\]
where $\Delta_{t,i,j}$ is independent of $c$ and thus independent of $\theta$. Multiplying by $q_t(i)$, summing over $i \in \calS$ and $j \neq i$, and integrating over $t$ yields
\[
\mathcal{L}_{\mathrm{KL}}(\theta) = \KLdiv(\hat{\Qmeas}\|\Pmeas^\theta) + \int_0^T \sum_{i \in \calS} q_t(i)\sum_{j\neq i}\Delta_{t,i,j}\,\dd t.
\]
Setting $C_{\mathrm{gap}} := \int_0^T \sum_{i \in \calS} q_t(i)\sum_{j\neq i}\Delta_{t,i,j}\,\dd t \ge 0$, under assumptions (i)--(iii), differentiation can be passed through the expectation and integral, so $\nabla_\theta \mathcal{L}_{\mathrm{KL}}(\theta) = \nabla_\theta \KLdiv(\hat{\Qmeas}\|\Pmeas^\theta)$. Adding a $\theta$-independent constant does not change the minimizers, and identical gradients imply identical stationary points.
\end{proof}

\begin{corollary}
\label{cor:L_cond_app}
Defining $K(a) := a(\log a - 1)$, the loss $\mathcal{L}_{\mathrm{KL}}(\theta)$ equals:
\begin{equation}
\begin{aligned}
\mathcal{L}_{\mathrm{cond}}(\theta) = \E_{t,\, x_0,\, i}\!\Bigl[
&\sum_{j \neq i} \lambdatheta_t(i) r^\theta_t(j \given i)
- \sum_{j \neq i} R_t(j,i) \frac{q_{t|0}(j \given x_0)}{q_{t|0}(i \given x_0)} \log \frac{\lambdatheta_t(i)\, r^\theta_t(j \given i)}{R_t(j,i)} \\
&+ \sum_{j \neq i} R_t(j,i)\, K\!\!\left(\frac{q_{t|0}(j \given x_0)}{q_{t|0}(i \given x_0)}\right)\Bigr].
\end{aligned}
\end{equation}
\end{corollary}

\begin{proof}
For fixed $(t,x_0,i,j)$, write
\[
\hat{R}_{t|0}(i,j\given x_0) = R_t(j,i)\,\frac{q_{t|0}(j\given x_0)}{q_{t|0}(i\given x_0)}
\quad\text{and}\quad
\Rthetamat_t(i,j) = \lambdatheta_t(i)\, r^\theta_t(j\given i).
\]
Expanding the Poisson KL density gives
\begin{equation}
\begin{aligned}
f\!\left(\hat{R}_{t|0}(i,j\given x_0), \Rthetamat_t(i,j)\right)
&= \hat{R}_{t|0}(i,j\given x_0)\log \frac{\hat{R}_{t|0}(i,j\given x_0)}{\Rthetamat_t(i,j)} - \hat{R}_{t|0}(i,j\given x_0) + \Rthetamat_t(i,j) \\
&= \Rthetamat_t(i,j) - \hat{R}_{t|0}(i,j\given x_0)\log \frac{\Rthetamat_t(i,j)}{R_t(j,i)} \\
&\quad + R_t(j,i)\,K\!\left(\frac{q_{t|0}(j\given x_0)}{q_{t|0}(i\given x_0)}\right).
\end{aligned}
\end{equation}
Summing over $j \neq i$ and taking expectation over $(t,x_0,i)$ yields $\mathcal{L}_{\mathrm{cond}}(\theta)$. Since $\mathcal{L}_{\mathrm{KL}}(\theta)$ is defined as the expectation of the same pairwise density $f$, this proves $\mathcal{L}_{\mathrm{cond}}(\theta)=\mathcal{L}_{\mathrm{KL}}(\theta)$.
\end{proof}

\label{app:proof_mdlm}
\begin{corollary}[Equivalence to MDLM, restated]
If the forward process is the masked (absorbing) process on $\calS = \mathcal{V} \cup \{m\}$, then for any valid noise schedule $\alpha_t$, the loss~\eqref{eq:L_cond} reduces to the MDLM cross-entropy loss~\eqref{eq:mdlm_equiv}.
\end{corollary}
\begin{proof}
We work directly with the full loss $\mathcal{L}_{\mathrm{cond}}(\theta)$ in~\eqref{eq:L_cond}, including its $\theta$-independent term. For the absorbing process,
\[
q_{t|0}(\cdot \given x_0) = \alpha_t \boldsymbol{x}_0 + (1-\alpha_t)\boldsymbol{e}_m,
\qquad
R_t(i,j) = \rho_t(\boldsymbol{e}_j^\top\boldsymbol{e}_m)\indicator_{\{i\neq m\}},
\qquad
\rho_t := -\frac{\alpha_t'}{\alpha_t}.
\]
Define the reverse coefficient
\[
\gamma_t := \rho_t\frac{\alpha_t}{1-\alpha_t} = -\frac{\alpha_t'}{1-\alpha_t} > 0.
\]
The $\theta$-dependent part of the integrand decomposes as
\[
\underbrace{-\sum_{j \neq i}\Rhatmat_t(i,j\given x_0)\log \Rthetamat_t(i,j)}_{\text{Term \textcircled{1}}} \;+\; \underbrace{\sum_{j \neq i}\Rthetamat_t(i,j)}_{\text{Term \textcircled{2}}} \;+\; \text{($\theta$-independent terms)},
\]
where $\Rhatmat_t(i,j\given x_0) = \gamma_t\,\boldsymbol{e}_j^\top(\boldsymbol{x}_0 - \boldsymbol{e}_m)\indicator_{\{i=m\}}$ and $\Rthetamat_t(i,j) = \gamma_t\,\boldsymbol{e}_j^\top(x_\theta - \boldsymbol{e}_m)\indicator_{\{i=m\}}$.

When $i \neq m$, both rates vanish, so the integrand is zero. For $i = m$, the sum runs over $j \neq m$, and orthogonality $\boldsymbol{e}_j^\top \boldsymbol{e}_m = 0$ simplifies the rates to $\Rhatmat_t(m,j\given x_0) = \gamma_t\,\boldsymbol{e}_j^\top \boldsymbol{x}_0$ and $\Rthetamat_t(m,j) = \gamma_t\,\boldsymbol{e}_j^\top x_\theta$.

Term \textcircled{1}:
\begin{align*}
\text{\textcircled{1}} &= -\sum_{j \neq m}\gamma_t(\boldsymbol{e}_j^\top \boldsymbol{x}_0)\log\bigl(\gamma_t\,\boldsymbol{e}_j^\top x_\theta\bigr)
= -\gamma_t\log\gamma_t\underbrace{\sum_{j\neq m}(\boldsymbol{e}_j^\top \boldsymbol{x}_0)}_{=\,1} - \gamma_t\sum_{j\neq m}(\boldsymbol{e}_j^\top \boldsymbol{x}_0)\log(\boldsymbol{e}_j^\top x_\theta) \\
&= -\gamma_t\log\gamma_t - \gamma_t\,\boldsymbol{x}_0^\top\log x_\theta,
\end{align*}
where we used that $x_0$ is one-hot on $\mathcal{V}$ so $\sum_{j\neq m}\boldsymbol{e}_j^\top \boldsymbol{x}_0 = 1$.

Term \textcircled{2}:
\[
\text{\textcircled{2}} = \sum_{j\neq m}\gamma_t\,\boldsymbol{e}_j^\top x_\theta = \gamma_t\underbrace{\sum_{j\neq m}\boldsymbol{e}_j^\top x_\theta}_{=\,1} = \gamma_t,
\]
using the normalization of $x_\theta$ on $\mathcal{V}$.

Combining: setting $C_t := -\gamma_t\log\gamma_t + \gamma_t$ (independent of $\theta$):
\[
\text{\textcircled{1}} + \text{\textcircled{2}} = -\gamma_t\,\boldsymbol{x}_0^\top\log x_\theta + C_t = \frac{\alpha_t'}{1-\alpha_t}\,\boldsymbol{x}_0^\top\log x_\theta + C_t.
\]
It remains to evaluate the $\theta$-independent terms in the full expression~\eqref{eq:L_cond}. For $i=m$, only the clean token contributes, so
\[
\sum_{j\neq m}\Rhatmat_t(m,j\given x_0)\log R_t(j,m)
+ \sum_{j\neq m}R_t(j,m)K\!\left(\frac{q_{t|0}(j\given x_0)}{q_{t|0}(m\given x_0)}\right)
= \gamma_t\log\rho_t + \rho_t K\!\left(\frac{\alpha_t}{1-\alpha_t}\right).
\]
Using $K(a)=a(\log a-1)$ and $\gamma_t=\rho_t\alpha_t/(1-\alpha_t)$, this becomes
\[
\gamma_t\log\rho_t + \gamma_t\left(\log\frac{\alpha_t}{1-\alpha_t}-1\right)
= \gamma_t\log\gamma_t - \gamma_t
= -C_t.
\]
Thus the $\theta$-independent term in~\eqref{eq:L_cond} exactly cancels the constant $C_t$ produced by Terms \textcircled{1} and \textcircled{2}. Taking the expectation over $(t,x_0,x_t)$ therefore yields~\eqref{eq:mdlm_equiv} with no additive constant.
\end{proof}

\begin{proposition}[Pointwise optimality, restated]
For fixed $t$ and $i$, define $g_{t,i}(\lambda, r) := \KLdiv^{\Poi}\!\bigl(\lambdahat_t(i) \| \lambda\bigr) + \lambdahat_t(i)\,\KLdiv^{\Cat}\!\bigl(\hat{r}_t(\cdot \given i) \| r\bigr)$.
If $\lambdahat_t(i) > 0$, then $g_{t,i}(\lambda, r) \ge 0$ with equality if and only if $\lambda = \lambdahat_t(i)$ and $r = \hat{r}_t(\cdot \given i)$. If $\lambdahat_t(i) = 0$, the unique minimizer in $\lambda$ is $\lambda = 0$ and the choice of $r$ is irrelevant.
\end{proposition}
\begin{proof}
Since both terms are non-negative, $g_{t,i}(\lambda, r) \ge 0$. For the Poisson KL, $\KLdiv^{\Poi}(\hat{\lambda} \| \lambda) = \hat{\lambda} \log \frac{\hat{\lambda}}{\lambda} - \hat{\lambda} + \lambda$. Taking the derivative with respect to $\lambda$ and setting it to zero: $-\hat{\lambda}/\lambda + 1 = 0$, giving $\lambda^* = \hat{\lambda}$. The second derivative $\hat{\lambda}/\lambda^2 > 0$ confirms this is a strict minimum, and $\KLdiv^{\Poi}(\hat{\lambda} \| \hat{\lambda}) = 0$. For the categorical KL, $\KLdiv^{\Cat}(\hat{r} \| r) \ge 0$ by Gibbs' inequality, with equality if and only if $r = \hat{r}$. When $\lambdahat_t(i) = 0$, the Poisson KL reduces to $\lambda \ge 0$ with minimum at $\lambda = 0$, and the categorical term vanishes since it is weighted by $\lambdahat_t(i) = 0$.
\end{proof}

\section{Additional Experimental Details}
\label{app:exp_details}

\textbf{Datasets.} We evaluate on two datasets: (1) TinyStories~\citep{eldan2023tinystories}, a corpus of simple short stories for language modeling; (2) OpenWebText (OWT)~\citep{radford2019language}, a large-scale web text corpus. Both are tokenized with the GPT-2 tokenizer (vocabulary size 50257) and truncated or padded to a maximum sequence length of 512.

\textbf{Architecture.} Neural CTMC uses a 12-layer DiT-style Transformer with hidden dimension 768, 12 attention heads, and a time-conditioning dimension of 128. The total parameter count is approximately 163M.

\begin{figure}[H]
\begin{minipage}[t]{0.48\textwidth}
\begin{algorithm}[H]
\caption{$\tau$-Leaping}
\label{alg:tau_leaping}
\begin{algorithmic}[1]
\STATE Set $\tau = \frac{T}{N}$ and sample $x_T \sim \mathrm{Uniform}(\calS)^L$
\FOR{$n = N$ \textbf{downto} $1$}
  \STATE $t \leftarrow n\tau$
  \STATE $\{(\lambda^\theta_{t,\ell},\, r^\theta_{t,\ell})\}_{\ell=1}^L \leftarrow \Phi_\theta(x_t, t)$
  \STATE Set $x_{t-\tau} \leftarrow x_t$
  \FOR{$\ell = 1$ \textbf{to} $L$}
    \STATE Sample $\Delta_\ell \sim \mathrm{Exp}(\lambda^\theta_{t,\ell})$
    \IF{$\Delta_\ell < \tau$}
      \STATE $x_{t-\tau}^{(\ell)} \sim r^\theta_{t,\ell}(\cdot \given x_t^{(\ell)})$
    \ENDIF
  \ENDFOR
\ENDFOR
\STATE \textbf{Return} $x_0$
\end{algorithmic}
\end{algorithm}
\end{minipage}
\hfill
\begin{minipage}[t]{0.48\textwidth}
\begin{algorithm}[H]
\caption{Euler}
\label{alg:euler}
\begin{algorithmic}[1]
\STATE Set $\tau = \frac{T}{N}$ and sample $x_T \sim \mathrm{Uniform}(\calS)$
\FOR{$n = N$ \textbf{downto} $1$}
  \STATE $t \leftarrow n\tau$
  \STATE $\bigl(\lambda^\theta_t(x_t),\, r^\theta_t(\cdot\given x_t)\bigr) \leftarrow \Phi_\theta(x_t, t)$
  \STATE Set $p_j \leftarrow \lambda^\theta_t(x_t)\cdot r^\theta_t(j\given x_t)\cdot\tau$ for $j \neq x_t$
  \STATE Set $p_{x_t} \leftarrow 1-\sum_{j\neq x_t}p_j$
  \STATE $x_{t-\tau} \sim \mathrm{Cat}\!\bigl(\{p_j\}_{j\in\calS}\bigr)$
\ENDFOR
\STATE \textbf{Return} $x_0$
\end{algorithmic}
\end{algorithm}
\end{minipage}
\end{figure}

\textbf{Baselines.} For TinyStories, we train MDLM and GIDD from their official source code using the same architecture and training budget as Neural CTMC. For OpenWebText, we use the official open-source released checkpoints and configurations of SEDD (682B training tokens), MDLM (262B), and GIDD (262B). The OWT comparison aligns backbone architecture, parameter scale, training-token budget for the equal-budget baselines, tokenization, sampling-step grid, number of generated samples, and Gemma2-9B scoring.

\textbf{Evaluation.} All methods are evaluated using the same protocol: we generate 1024 unconditional samples per configuration and score them with a pretrained \textbf{Gemma2-9B} model to obtain generative perplexity (PPL). This measures sample quality under a common external language model, rather than the ELBO-based perplexity reported in some prior work. For TinyStories, we report PPL over training epochs with 50 sampling steps. For OpenWebText, we sweep the number of sampling steps over $\{16, 32, 64, 128\}$ at the final checkpoint.

\textbf{Sampling.} We evaluate two sampling schemes: Euler (Algorithm~\ref{alg:euler}) and $\tau$-leaping (Algorithm~\ref{alg:tau_leaping}). The number of sampling steps is 50 for TinyStories training curves and varies from 16 to 128 for the OpenWebText step-sweep experiment.

\textbf{Scaling law.} To assess whether the benefits of Neural CTMC persist as compute increases, we conduct a scaling-law study on OpenWebText along three axes: model size, training data, and total training FLOPs. We sweep model sizes $\{\text{tiny (80.1M)}, \text{small (169.7M)}, \text{base (424.5M)}\}$ and token budgets $\{100\text{M}, 300\text{M}, 1\text{B}\}$, with three random seeds per configuration; the metric is the best validation loss. As shown in Figure~\ref{fig:scaling_law}, the validation loss decreases monotonically with both parameter count and training tokens, and a clear compute-optimal frontier emerges: the best run, \texttt{base} at the 1B-token budget, attains a validation loss of $4.150$. The fitted power-law exponents on the compute frontier ($\alpha \approx 0.16$, $R^2 \approx 0.93$) indicate that Neural CTMC follows scaling behavior consistent with mature autoregressive and diffusion-based language models, suggesting that the proposed parameterization scales gracefully and that further gains can be expected at larger compute budgets.

\begin{figure}[H]
    \centering
    \includegraphics[width=\linewidth]{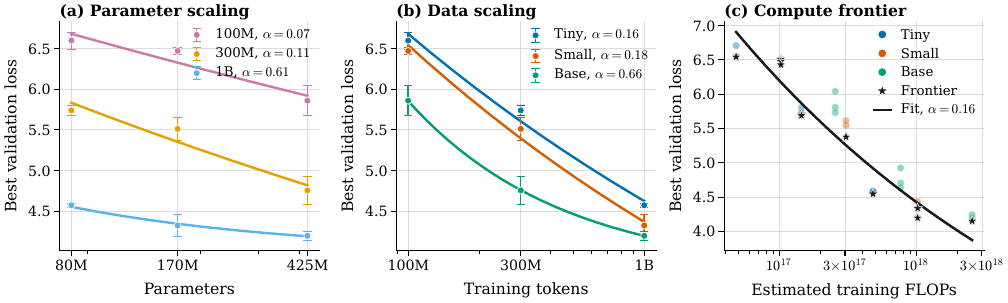}
    \caption{Scaling-law results for Neural CTMC on OpenWebText. Validation loss decreases with model size, training tokens, and total training FLOPs, exhibiting a clear power-law trend on the compute-optimal frontier.}
    \label{fig:scaling_law}
\end{figure}

\section{Self-correction}
\label{app:self_correction}

We additionally consider a lightweight self-correction step that can be applied to a partially denoised sequence during sampling. For sequence data, write $x_t=(x_t^{(1)},\ldots,x_t^{(L)})\in\calS^L$. The trained model $\Phi_\theta$ is evaluated on the full sequence and produces the Neural CTMC outputs from~\eqref{eq:neural_param} at each position. The step below uses these outputs to recover an auxiliary clean-token distribution and then updates only the most confident proposed correction.

\paragraph{Recovering a clean-token distribution.}
For the uniform forward process in~\eqref{eq:cumulative_transition} with $\pi=\frac{1}{S}\mathbf{1}$, the noised distribution associated with a clean-token distribution $p_0$ has the form $q=\alpha_t p_0+\beta_t\pi$. At position $\ell$, let $i=x_t^{(\ell)}$ and let $\tilde q^\theta_{t,\ell}$ denote the model-implied noised distribution recovered from $\Phi_\theta(x_t,t)$. Using the reverse-rate factorization,
\begin{equation}
\lambdatheta_t(i)
    \approx \frac{1}{S\alpha_t}\frac{1-\tilde q^\theta_{t,\ell}(i)}
    {\tilde q^\theta_{t,\ell}(i)}, \qquad
\lambdatheta_t(i) r^\theta_t(j\given i)
    \approx \frac{1}{S\alpha_t}\frac{\tilde q^\theta_{t,\ell}(j)}
    {\tilde q^\theta_{t,\ell}(i)}, \quad j\neq i.
\end{equation}
Thus we form
\begin{equation}
\tilde q^\theta_{t,\ell}(i)
    = \frac{1}{1+S\alpha_t\lambdatheta_t(i)}, \qquad
\tilde q^\theta_{t,\ell}(j)
    = \bigl(1-\tilde q^\theta_{t,\ell}(i)\bigr)r^\theta_t(j\given i), \quad j\neq i,
\end{equation}
and map back to a clean-token distribution by
\begin{equation}
\hat p^\theta_{0,t,\ell}
    = \operatorname{Normalize}\!\left(\left[
    \frac{\tilde q^\theta_{t,\ell}-\beta_t\pi}{\alpha_t}
    \right]_+\right),
\end{equation}
where $[\cdot]_+$ denotes elementwise clipping at zero. For temperature $\tau_{\mathrm{sc}}>0$, define
\begin{equation}
\operatorname{Temp}(p;\tau_{\mathrm{sc}})_j
    := \frac{p_j^{1/\tau_{\mathrm{sc}}}}{\sum_{a\in\calS}p_a^{1/\tau_{\mathrm{sc}}}}.
\end{equation}

\begin{algorithm}[H]
\caption{Self-correction}
\label{alg:self_correction}
\begin{algorithmic}[1]
\REQUIRE Partially denoised sequence $x_t=(x_t^{(1)},\ldots,x_t^{(L)})$; trained Neural CTMC $\Phi_\theta$; noise level $t$; temperature $\tau_{\mathrm{sc}}$; maximum updates $K$
\ENSURE Corrected sequence $x_t$
\FOR{$k=1$ \textbf{to} $K$}
  \FOR{$\ell=1$ \textbf{to} $L$}
    \STATE Let $i \leftarrow x_t^{(\ell)}$ and compute $\bigl(\lambdatheta_t(i),\, r^\theta_t(\cdot\given i)\bigr)$ from $\Phi_\theta(x_t,t)$ at position $\ell$
    \STATE Recover $\hat p^\theta_{0,t,\ell}$ from $\bigl(\lambdatheta_t(i),\, r^\theta_t(\cdot\given i)\bigr)$ using the equations above
    \STATE Sample $\tilde x_0^{(\ell)} \sim \Cat\!\left(\operatorname{Temp}\!\left(\hat p^\theta_{0,t,\ell};\tau_{\mathrm{sc}}\right)\right)$
  \ENDFOR
  \STATE $\mathcal{I} \leftarrow \{\ell \in \{1,\ldots,L\}: \tilde x_0^{(\ell)} \neq x_t^{(\ell)}\}$
  \IF{$\mathcal{I}=\varnothing$}
    \STATE \textbf{break}
  \ENDIF
  \STATE $\ell^\star \leftarrow \arg\max_{\ell\in\mathcal{I}} \hat p^\theta_{0,t,\ell}\!\left(\tilde x_0^{(\ell)}\right)$
  \STATE $x_t^{(\ell^\star)} \leftarrow \tilde x_0^{(\ell^\star)}$
\ENDFOR
\STATE \textbf{Return} $x_t$
\end{algorithmic}
\end{algorithm}

\paragraph{Quality evaluation.}
We evaluate the same post-sampling setting on OpenWebText using a common-sample protocol. We first generate $32$ raw samples with the Neural CTMC sampler and then apply each self-correction method to exactly the same raw texts. For GIDD, we use the official released checkpoints with $p_u\in\{0.0,0.1,0.2\}$ and the self-correction temperatures used in the GIDD study: $\tau_{\mathrm{sc}}=0.1$ for $p_u=0.0,0.1$ and $\tau_{\mathrm{sc}}=0.5$ for $p_u=0.2$. For Neural CTMC, we evaluate $\tau_{\mathrm{sc}}\in\{0.1,0.5\}$ with the same update rule in Algorithm~\ref{alg:self_correction}. Outputs are judged by GPT-5.2 on a $1$--$10$ scale along five text-quality axes. Table~\ref{tab:self_correction_quality} shows that Neural CTMC self-correction is competitive with GIDD self-correction, with the two strongest variants splitting the per-axis wins. Neural CTMC at $\tau_{\mathrm{sc}}=0.1$ attains the best Grammaticality ($4.40$, $+12.4\%$, $>2\sigma$), Factuality ($3.09$, $+6.2\%$), and Creativity ($5.63$, $+8.0\%$, $>2\sigma$) among all variants, while the hybrid GIDD checkpoint with $p_u=0.2$ and $\tau_{\mathrm{sc}}=0.5$ achieves the best Clarity ($4.00$, $+5.8\%$, $>2\sigma$) and Writing style ($3.88$, $+8.4\%$, $>2\sigma$). The two methods are essentially tied on the five-axis average ($4.14$ for Neural CTMC at $\tau_{\mathrm{sc}}=0.1$ vs.\ $4.13$ for GIDD at $p_u=0.2$), and we note that GIDD ($p_u=0.2$) shows $>2\sigma$ improvements on four of the five axes versus two for Neural CTMC ($\tau_{\mathrm{sc}}=0.1$), so the comparison is best read as complementary strengths rather than a uniform win.

\begin{table}[H]
\centering
\scriptsize
\setlength{\tabcolsep}{3pt}
\resizebox{\linewidth}{!}{%
\begin{tabular}{lccccc}
\toprule
Model & Clarity & Grammaticality & Factuality & Writing style & Creativity \\
\midrule
Original sample & 3.78 & 3.91 & 2.91 & 3.58 & 5.21 \\
\midrule
GIDD self-corr. ($p_u=0.0, \tau_{\mathrm{sc}}=0.1$) & 3.19 {\color{red!70!black}(-15.7\%)*} & 3.01 {\color{red!70!black}(-23.2\%)**} & 3.03 {\color{green!45!black}(+4.3\%)} & 3.08 {\color{red!70!black}(-13.9\%)*} & 4.57 {\color{red!70!black}(-12.3\%)*} \\
\midrule
GIDD self-corr. ($p_u=0.1, \tau_{\mathrm{sc}}=0.1$) & 3.81 {\color{green!45!black}(+0.6\%)} & 4.00 {\color{green!45!black}(+2.2\%)} & 3.08 {\color{green!45!black}(+5.7\%)} & 3.63 {\color{green!45!black}(+1.5\%)} & 5.19 {\color{red!70!black}(-0.4\%)} \\
\midrule
GIDD self-corr. ($p_u=0.2, \tau_{\mathrm{sc}}=0.5$) & \textbf{4.00} {\color{green!45!black}(+5.8\%)*} & 4.25 {\color{green!45!black}(+8.6\%)*} & 2.98 {\color{green!45!black}(+2.5\%)} & \textbf{3.88} {\color{green!45!black}(+8.4\%)*} & 5.54 {\color{green!45!black}(+6.2\%)*} \\
\midrule
Neural CTMC self-corr. ($\tau_{\mathrm{sc}}=0.1$) & 3.81 {\color{green!45!black}(+0.6\%)} & \textbf{4.40} {\color{green!45!black}(+12.4\%)*} & \textbf{3.09} {\color{green!45!black}(+6.2\%)} & 3.77 {\color{green!45!black}(+5.4\%)} & \textbf{5.63} {\color{green!45!black}(+8.0\%)*} \\
\midrule
Neural CTMC self-corr. ($\tau_{\mathrm{sc}}=0.5$) & 3.82 {\color{green!45!black}(+1.1\%)} & 4.06 {\color{green!45!black}(+3.8\%)} & 3.01 {\color{green!45!black}(+3.4\%)} & 3.74 {\color{green!45!black}(+4.5\%)} & 5.57 {\color{green!45!black}(+6.8\%)*} \\
\bottomrule
\end{tabular}%
}
\caption{Quality ratings after correcting the same raw OpenWebText samples. Percent changes are relative to the original samples; * denotes $>2\sigma$ and ** denotes $>5\sigma$ paired difference across seeds.}
\label{tab:self_correction_quality}
\end{table}

\paragraph{Examples.}
To illustrate what the step changes in practice, we apply self-correction to an OpenWebText sample generated by the Euler sampler with $128$ denoising steps. We hold the base sample and correction seed fixed, and vary only the update budget $K$. Table~\ref{tab:self_correction_examples} highlights the edited tokens before and after correction. Small budgets tend to make local grammatical or lexical repairs, while larger budgets allow more aggressive changes; this matches the interpretation of $K$ as a controllable post-sampling correction strength.

\begingroup
\scriptsize
\setlength{\tabcolsep}{3pt}
\renewcommand{\arraystretch}{1.08}
\setlength{\LTcapwidth}{\linewidth}
\begin{longtable}{@{}>{\raggedright\arraybackslash}p{0.49\linewidth}|>{\raggedright\arraybackslash}p{0.49\linewidth}@{}}
\caption{Self-correction examples for update budgets $K\in\{4,8,16,24\}$. The base sample is generated using Euler sampling with $128$ steps.}
\label{tab:self_correction_examples}\\
\hline
\rowcolor{headergray}\multicolumn{2}{c}{\textit{Example 1: self-correction max\_updates=4}}\\
\hline
\textbf{Original sample} & \textbf{After self-correction}\\
\hline
\endfirsthead
\caption[]{Self-correction examples (continued).}\\
\hline
\textbf{Original sample} & \textbf{After self-correction}\\
\hline
\endhead
\hline
\multicolumn{2}{r}{\textit{Continued on next page}}\\
\endfoot
\hline
\endlastfoot
Martin Samuel, who changed his face because she changed his face, is drawn to profound insight by writer John Lewis. & Martin Samuel, who changed his face because she changed his face, is drawn to profound insight by writer John Lewis.\\
\hline
His first experience with cannabis in 2008 and 2009, and later known as Boone Half the Sole, \oldword{came} at the age of 5, when he returned from high school and ingested a pipe. & His first experience with cannabis in 2008 and 2009, and later known as Boone Half the Sole, \newword{was} at the age of 5, when he returned from high school and ingested a pipe.\\
\hline
The research program lists marijuana as a federal harm reduction drug, and now it has \oldword{been} legal at a high recreational level, and it has surged in all US states. & The research program lists marijuana as a federal harm reduction drug, and now it has \newword{become} legal at a high recreational level, and it has surged in all US states.\\
\hline
Samuel is celebrating his return from high school and proposes society deserves to stop using cannabis as an excuse, and the exciting and welcome changes, during brief years of life, that are largely reflected in a drug. & Samuel is celebrating his return from high school and proposes society deserves to stop using cannabis as an excuse, and the exciting and welcome changes, during brief years of life, that are largely reflected in a drug.\\
\hline
Recently President Obama said that cannabis was a ``vulnerability'' in a single age, to which ``I believe a novel and means; all forms of human life should not depend one another on marijuana''. & Recently President Obama said that cannabis was a ``vulnerability'' in a single age, to which ``I believe a novel and means; all forms of human life should not depend one another on marijuana''.\\
\hline
Samuel says, `` rivalry old man, who you known as before, you are a very fond and wonderful old person. & Samuel says, `` rivalry old man, who you known as before, you are a very fond and wonderful old person.\\
\hline
``I \oldword{suffered} a rare [cancer] in \oldword{extremely} good health and I eat with it on the streets - and with my bones, every day it's very strong.'' & ``I \newword{have} a rare [cancer] in \newword{very} good health and I eat with it on the streets - and with my bones, every day it's very strong.''\\
\hline
He also is a man, who remembers everything every day, and is pivotal with keeping his life ``just a completely ordinary life''. & He also is a man, who remembers everything every day, and is pivotal with keeping his life ``just a completely ordinary life''.\\
\hline
\rowcolor{headergray}\multicolumn{2}{c}{\textit{Example 2: self-correction max\_updates=8}}\\
\hline
\textbf{Original sample} & \textbf{After self-correction}\\
\hline
Martin Samuel, who changed his face because she changed his face, is drawn to profound insight by writer John Lewis. & Martin Samuel, who changed his face because she changed his face, is drawn to profound insight by writer John Lewis.\\
\hline
His first experience with cannabis in 2008 and 2009, and later known as Boone Half the Sole, \oldword{came} at the age of 5, when he returned from high school and ingested a pipe. & His first experience with cannabis in 2008 and 2009, and later known as Boone Half the Sole, \newword{was} at the age of 5, when he returned from high school and ingested a pipe.\\
\hline
The research program lists marijuana as a federal harm reduction drug, and now it has \oldword{been} legal at a high recreational level, and it has surged in all US states. & The research program lists marijuana as a federal harm reduction drug, and now it has \newword{become} legal at a high recreational level, and it has surged in all US states.\\
\hline
Samuel is celebrating his return from high school and proposes society deserves to stop using cannabis as an excuse, \oldword{and} the exciting and welcome changes, during brief years of life, that are largely reflected in \oldword{a} drug. & Samuel is celebrating his return from high school and proposes society deserves to stop using cannabis as an excuse, \newword{for} the exciting and welcome changes, during brief years of life, that are largely reflected in \newword{the} drug.\\
\hline
Recently President Obama said that cannabis was a ``vulnerability'' in a single age, to which ``I believe a novel and means; all forms of human life should not depend one another on marijuana''. & Recently President Obama said that cannabis was a ``vulnerability'' in a single age, to which ``I believe a novel and means; all forms of human life should not depend one another on marijuana''.\\
\hline
Samuel says, `` rivalry old man, who you known as before, you are a very fond and wonderful old person. & Samuel says, `` rivalry old man, who you known as before, you are a very fond and wonderful old person.\\
\hline
``I \oldword{suffered} a rare [cancer] in \oldword{extremely} good health and I eat with it on the streets - and with my bones, every day it's very strong.'' & ``I \newword{have} a rare [cancer] in \newword{very} good health and I eat with it on the streets - and with my bones, every day it's very strong.''\\
\hline
He also is a man, who remembers everything every day, and is pivotal with \oldword{keeping} his life ``just a completely \oldword{ordinary} life''. & He also is a man, who remembers everything every day, and is pivotal with \newword{making} his life ``just a completely \newword{normal} life''.\\
\hline
\rowcolor{headergray}\multicolumn{2}{c}{\textit{Example 3: self-correction max\_updates=16}}\\
\hline
\textbf{Original sample} & \textbf{After self-correction}\\
\hline
Martin Samuel, who changed his face because she changed his face, is \oldword{drawn} to \oldword{profound} insight by writer John Lewis. & Martin Samuel, who changed his face because she changed his face, is \newword{inspired} to \newword{his} insight by writer John Lewis.\\
\hline
His first experience with cannabis in 2008 and 2009, and later known as Boone Half the Sole, \oldword{came} at the age of 5, when he returned from high school and ingested a pipe. & His first experience with cannabis in 2008 and 2009, and later known as Boone Half the Sole, \newword{was} at the age of 5, when he returned from high school and ingested a pipe.\\
\hline
The research program lists marijuana as a federal harm reduction drug, \oldword{and} now it has \oldword{been} legal at \oldword{a} high recreational level, and it has surged in all US states. & The research program lists marijuana as a federal harm reduction drug, \newword{but} now it has \newword{become} legal at \newword{the} high recreational level, and it has surged in all US states.\\
\hline
Samuel is celebrating his return from high school and proposes society deserves to stop using cannabis as an excuse, \oldword{and} the exciting and welcome changes, \oldword{during} brief years of life, that are largely reflected in \oldword{a} drug. & Samuel is celebrating his return from high school and proposes society deserves to stop using cannabis as an excuse, \newword{for} the exciting and welcome changes, \newword{in} brief years of life, that are largely reflected in \newword{the} drug.\\
\hline
Recently President Obama said that cannabis was a ``vulnerability'' in a \oldword{single} age, to which ``I believe a novel and means; all forms of human life should not depend one another on marijuana''. & Recently President Obama said that cannabis was a ``vulnerability'' in a \newword{young} age, to which ``I believe a novel and means; all forms of human life should not depend one another on marijuana''.\\
\hline
Samuel says, `` rivalry old man, who you known as before, you are a very fond and wonderful old person. & Samuel says, `` rivalry old man, who you known as before, you are a very fond and wonderful old person.\\
\hline
``I \oldword{suffered} a rare [cancer] in \oldword{extremely} good health and I eat with it on the \oldword{streets} \oldword{-} and with my bones, every day it's very strong.'' & ``I \newword{have} a rare [cancer] in \newword{very} good health and I eat with it on the \newword{streets,} and with my bones, every day it's very strong.''\\
\hline
He also is a man, who remembers everything every day, and is pivotal \oldword{with} \oldword{keeping} his life ``just a completely \oldword{ordinary} life''. & He also is a man, who remembers everything every day, and is pivotal \newword{in} \newword{making} his life ``just a completely \newword{normal} life''.\\
\hline
\rowcolor{headergray}\multicolumn{2}{c}{\textit{Example 4: self-correction max\_updates=24}}\\
\hline
\textbf{Original sample} & \textbf{After self-correction}\\
\hline
Martin Samuel, who changed his face because she changed his face, is \oldword{drawn} to \oldword{profound} insight by writer John Lewis. & Martin Samuel, who changed his face because she changed his face, is \newword{inspired} to \newword{his} insight by writer John Lewis.\\
\hline
His first experience with cannabis in 2008 \oldword{and} 2009, and later known as Boone Half the Sole, \oldword{came} at the age of 5, when he returned from high school and ingested a pipe. & His first experience with cannabis in 2008 \newword{in} 2009, and later known as Boone Half the Sole, \newword{was} at the age of 5, when he returned from high school and ingested a pipe.\\
\hline
The \oldword{research} \oldword{program} lists marijuana as a federal harm reduction drug, \oldword{and} now it has \oldword{been} legal at \oldword{a} high recreational level, and it \oldword{has} \oldword{surged} in all US states. & The \newword{FDA} \newword{still} lists marijuana as a federal harm reduction drug, \newword{but} now it has \newword{become} legal at \newword{the} high recreational level, and it \newword{is} \newword{legal} in all US states.\\
\hline
Samuel is celebrating his return from high school and \oldword{proposes} society deserves to stop using cannabis as an excuse, \oldword{and} the exciting and welcome changes, \oldword{during} brief years of life, that are largely reflected in \oldword{a} drug. & Samuel is celebrating his return from high school and \newword{says} society deserves to stop using cannabis as an excuse, \newword{for} the exciting and welcome changes, \newword{in} brief years of life, that are largely reflected in \newword{the} drug.\\
\hline
Recently President Obama said that cannabis was a ``vulnerability'' \oldword{in} a \oldword{single} age, to which ``I believe a novel and means; all forms of human life should not depend one another on marijuana''. & Recently President Obama said that cannabis was a ``vulnerability'' \newword{at} a \newword{young} age, to which ``I believe a novel and means; all forms of human life should not depend one another on marijuana''.\\
\hline
Samuel says, `` rivalry old man, who you known as before, you are a very fond and wonderful old person. & Samuel says, `` rivalry old man, who you known as before, you are a very fond and wonderful old person.\\
\hline
``I \oldword{suffered} a rare [cancer] \oldword{in} \oldword{extremely} good health and I eat with it on the \oldword{streets} \oldword{-} and with my bones, every day it's very strong.'' & ``I \newword{have} a rare [cancer] \newword{and} \newword{very} good health and I eat with it on the \newword{streets,} and with my bones, every day it's very strong.''\\
\hline
He also is a man, who remembers everything every day, and is pivotal \oldword{with} \oldword{keeping} his life ``just a completely \oldword{ordinary} life''. & He also is a man, who remembers everything every day, and is pivotal \newword{in} \newword{making} his life ``just a completely \newword{normal} life''.\\
\end{longtable}
\endgroup

\newpage
\section{Generated Samples}
\label{app:samples}

We present generated samples from Neural CTMC, shown without post-processing.

\subsection{Image Generation}
\label{app:samples_mnist}

\begin{figure}[H]
\centering
\begin{subfigure}[t]{0.42\textwidth}
    \centering
    \includegraphics[width=\linewidth]{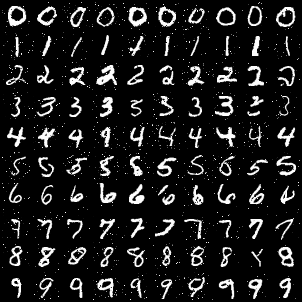}
    \caption{Euler (128 steps)}
\end{subfigure}
\hfill
\begin{subfigure}[t]{0.42\textwidth}
    \centering
    \includegraphics[width=\linewidth]{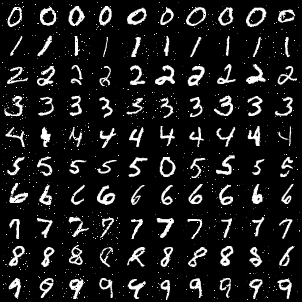}
    \caption{$\tau$-leaping (128 steps)}
\end{subfigure}
\caption{Unconditional MNIST samples generated by Neural CTMC at epoch 80 with 128 sampling steps. Each row corresponds to one digit class (0--9), with 10 samples per class.}
\label{fig:mnist}
\end{figure}

\subsection{TinyStories}
\label{app:samples_tinystory}

\noindent\textbf{Sample 1.}\quad
One day, a girl named Amy went to the park to play. She saw a van full of yummy snacks. Amy's friend Tom wanted to buy some from the store. Amy and Tom went to the store, but they didn't sell the snacks. There were still many things, and Amy was sad.
Then they saw a lady in the park. She saw how sad Amy was. The lady said ``Amy, why don't you sell the snacks?'' Amy thought for a moment and said, ``Let's tell the other boy to ask.'' They walked to a boy and said, ``We did this in the shop. Can we buy the food?''
The boy sold the food and Tom was happy. He said ``thank you, Tom!'' Amy smiled and said, ``You're welcome Tom! We shared a story and sometimes it's good to tell questions, it can make others happy.'' Amy and Tom became the best of friends.

\medskip
\noindent\textbf{Sample 2.}\quad
Once upon a time there was a stubborn little boy. He liked to play all day, but one day he found a whistle. He decided to pick it up and blow it. He wiggled around and his head started to blow a tune. He was so happy and he was having so much fun.
He was playing when a wise old owl came along. The owl asked the whistle: ``Dumb, you have some music coming from when you work up close. Watch the practice''.
The stubborn little boy listened as the owl had said earlier. He just listened to the music and then opened his mouth and blowing it in the air.
The whistle was so happy he could whistle louder, but eventually he started to forget how stubborn he was.

\medskip
\noindent\textbf{Sample 3.}\quad
Lily and Tom are at the park with their mom. They see a slide and want to go on the big slide. So they take turns to go first. Lily goes first and sits on the slide fast. She says, ``Look, I did it, I had fun!''
Tom slides down again and again. Lily is sad and angry. She walks to the bench and says, ``Take it back! The slide is not fair with all the fun!'' But Tom says, ``No, mom! Go away, this is my slide!''
Mom sees what is happening and comes to the bench. She sees Lily and Tom fighting. She is not happy. She says to them, ``You are naughty kids. You cannot play on the slide. You have to ask before you play. The slide is a park, it is a place to play. You have the chance.''
Lily and Tom feel sorry. They say, ``I'm sorry, mom. I was wrong. We will be nice next time. We will wait for you and you will be happy.'' They decide to go on the slide together. They climb the steps slowly. They make it to the top.
Mom smiles. She says, ``Good job, kids. You are brave. Let's go home now. Lunch time.'' Lily and Tom hug mom. They are happy again.

\medskip
\noindent\textbf{Sample 4.}\quad
Once upon a time, there was a boy called Timmy. Tommy loved to play outside, so he loaded rocks with him. Tommy especially loved playing in the sand when he was outside.
One night, Timmy got so light, he fell asleep by the sea. He ran to the beach to take his shovels in the sand.
When he came to the beach, though, he encountered fear. He saw a alive little puppy and that it seemed to be stuck in the sea. Timmy's heart sank and ran with his shovel to help the pup.
He worriedly dug free the little pup, as it couldn't breathe properly. The little puppy wagged its tail and was looking back at Tommy in his eyes.
From then on, Timmy never wanted to go back to the beach. He wanted to play and have a new furry friend to explore. The End.

\medskip
\noindent\textbf{Sample 5.}\quad
Once upon a time, there was a little girl who liked to touch things. She had a lot of questions and always wondered where she could go because she liked to explore nature.
One day in the forest, she came across something very fragile. It was a small button and looked very fascinating. She touched it out and realized it hurt her.
The little girl looked around and saw unknown things that she had never seen before. She was so excited as she touched it to show her friends. She quickly ran back home and proudly showed them the fragile button.
Her one of her friends said, ``That is amazing! That looks fragile. We should keep extra care in the future.''
The little girl smiled and said ``I know, what friendship will come with?''
Just that, the little girl and her friends had a wonderful time in the forest!

\medskip
\noindent\textbf{Sample 6.}\quad
Once upon a time there lived a girl called Mia. She was three years old and had big ears. All the little kids near the beach, Mia liked to whistle with her lip in the sandpit.
One day, Mia put on her swim suit and went to the beach. She looked around and saw a little box full of toys in it. Mia picked up the toys and tried to shout. Nothing she could hear, Mia tried and tried, but she couldn't.
Then, Mia heard a noise behind her. She turned and saw a man in a funny hat was walking towards her. He said, ``hello Mia!'' Mia was embarrassed but she decided not to be scared, and she managed to pay down the sound of whistle.
The man came out of the car and approached them both. Mia said, ``I'm Mia''. The man smiled and said, ``I'm Henry''. Mia's mom was so happy to meet Mia and the man. They all became friends and had a lot of fun. Mia even bought a new toy. Mia smiled and waved goodbye as the man left, feeling Mia would never forget her day.

\medskip
\noindent\textbf{Sample 7.}\quad
Jack woke ready today one morning. He put on his hat, and on his lovely shoes. He grabbed his lunch bag because he knew he was going.
At the park, Jack saw a big tank that he had never seen before. He was so excited and he ran over to see the tank, but when he got the tank down in front of him, he was very sad. He saw worms, bugs, and shout. But then a friendly voice shouted, ``Don't worry Jack. I can help you''.
Jack looked up and saw a friend. It was his friend John. John said, ``Let's go explore the tank together!'' So they went in, and Jack was so happy.
Jack and his friend explored together all night and eventually they discovered a big aquarium. They played games together and saw fish swimming in the tank. Jack was so happy.
Jack started it was time to go home, but instead he said ``I want to be friends?'' John just smiled and hugged his best friend. They both ran back to the park and they had a great day.

\medskip
\noindent\textbf{Sample 8.}\quad
Lily and Ben were twins and liked to play outside with their toys. They had a red cart that they always filled with tape and toys. They liked to make roads and houses with blocks and pretend they were lions and horses.
One day, they decided to go with their cart to the park. They asked their mom to let them have fun. She said yes and let them. They pushed their cart on the road, holding hands, jumping and laughing.
Ben was driving the cart, but Lily was faster. She saw that their cart had a big stick that had been cut off. She tried to cut it, but it was too tight. She pushed it off and ran away.
Ben said that was not fair and ran after it. They both took the cart with all their might. But they did not see a big puddle of mud behind. They did not hear how mom was coming to pull them out.
Their mom was not happy. She made them wash the mud and the cart. She told them to take away all their toys and go home. She told them to say sorry and share their toys, and never to play with their cart for a week.
Ben and Lily felt sad and sorry. They realized they did not mean to be polite. They wished they had never fought over the stick again. They wished they had been more polite.

\medskip
\noindent\textbf{Sample 9.}\quad
Once upon a time there was a bear and a rabbit in their den in the forest. In the forest there was a big cake. The bear and rabbit were curious about the cake so they decided to go.
Slowly, they bounced and skipped on their way. They were starting to feel very happy.
Soon they realized they were in a place to get to the cake.
``Where are we going?'' shouted the tree.
The rabbit and bear looked around.
Suddenly a wise voice said back, ``Why won't you return to a group of humans? They will compete for the best cake at the competition!''
The bear in the forest cheered. The bear was so excited and wanted to join the competition. Together they went on the stage and cheered as they won.

\medskip
\noindent\textbf{Sample 10.}\quad
Once upon a time there was a curious frog. He lived in a green and peaceful forest.
One day, the frog decided to go for a swim. He went to the pond but his body settled down the water. He liked to see the water and decided to take a closer look. But afterwards all the way closer, he came across a keeper that was causing the water.
The keeper didn't like having the water frog along the pond so he agreed to arrest him. The frog was sad but he was determined to give himself a try. He jumped and splashed with the animals for a long time.
Suddenly, something wonderful happened. The gust of wind blew very hard and the little frog was blown away from the pond. The frog was so happy and quickly swam away. The other animals watched him to drink the pond from above, happy he had helped his friends.

\subsection{OpenWebText}
\label{app:samples_owt}

\noindent\textbf{Sample 1.}\quad
Let's examine Golden State's next in this year-end battle: allowing people aspiring to buy health-care policies from companies to enjoy access to one for everyone.
The commission whose leading role in acquiring health insurance is considering that an exchange delivered from state over state to Golden State be moved to the new opportunity to compete to policies denied consumers by government --- resulting in the health insurance industry's loss of interest --- through 2017. While that leaves Sacramento residents looking for an option, a free, tangible savings account appears insufficient because they don't think it could help any of them.
``If the general democracy was against it, we would be reluctant to get our health insurance in that state and go through that state,'' Tourer said. Now that leaves an alternative: ``Making sure you don't out and buy insurance and anything about our state we love is equity and fairness and everything else, now we have to decide whether not for you and you want to have subsidies and or not,'' he said.
More close-to-your-time credit or financial security is what the current system does. But that financial stability has recently led to many types of miners copying and rewriting health insurance policies, as well as an increasing interest in the markets becoming the norm. That, unlike some things in the health care markets, are terms that certainly aren't too complicated in fact.
The federal government can explain why and how it changed years ago. Of course, the new Consumer Financial Protection Bureau report provided that some of the support changed came from the same agency that had commissioned a statewide campaign to force the 10-percent health care burden toward poor and rich Southern California. It argues that it means that the bottom line comes in with those of our new neighbors.

\medskip
\noindent\textbf{Sample 2.}\quad
After taking off a plane Monday morning, Alex Campbell met with the coach himself Monday night, as the Flames haven't played under Campbell this year. In addition to autographs by former Flames-line player and right-arm captain Alexander Grad (who was the Stanley Cup team's affiliate with hockey team Arizona last year) we've taken a look at the life of city's top scorer, who moves the third straight.
``I think we end up a pretty good overall record in a team. Really here, during that period when people come to a wedding or dinner, who take in season trophies, he does play remarkably well overall,'' said Westberg, who is coach in the North Central Lakes Helacs.
Former GM Simon Keller spent an extended 20 years in the NHL but has been slowly to build into a guy as an 18-year-old 13 years. He's pretty well engaged and has a good story. He began to really early grow as a person, and was definitely a contributor in a few of them before he went to the Devils in 2004.
That fierce sense of humor is useful to those in our organization as well as him. ``That's cognac stuff. And we are a great, centered person. It needs some energy and variety, he's great. And he's supposed to be a proper player. But we might deliver some tactical stuff.''

\medskip
\noindent\textbf{Sample 3.}\quad
As Japan's tech industry has coalesced their startup talent and created so much future there will be computer generated labor for entrepreneurs to do in Japanese. Concern among those around the country is about how best to train workers to complex tasks such as to for goods during limited shifts.
For some resigned, tech employees in this sector have been pushing to increase the labor base to more than expected, which could help boost corporations' share of corporate workers from a projected low 600,000 a year, according to USA Today, with some regulations to protect businesses from worker making human workforce.
Long-time automation --- as a result, the controversy surrounding an AI company by the Open House artificial intelligence Institute called Magnii, based on a capitalists' vision, has been heavily development by Nune Labs and its Tokyo headquarters. Though Tokyo, with its 11 million population, features the country's top job market, over a million Japanese entrepreneurs lack any longer work.
Nvo Labs had a disparate interest in AI that led to the head of police department Yumsuki backing away with Niigueu in fall of 2013. The Lab has committed to introducing autonomy and run AI development programs at warehouse laboratories in villages, comparable to an ITTA cooperative. She also continued to pursue efforts to install local government staff that can better handle high-end robots, like the AITs.

\medskip
\noindent\textbf{Sample 4.}\quad
Chief thing here today from my employer, Barlan Yblin. You've got to be the newest newcomer to the limelight, Bernie Sanders.
I've long been paying attention to the Sanders' argument too much. As CNBC reports, the American economy has lost roughly \$3.6 million worth of jobs per month in June, or \$37 million --- making up the rest, such as child care and childcare for mental veterans. Sanders seems to be not making claims that have put the economy ``under far greater strain.''
Case: In 2009, since is fourth year, a data center report showed the net worth of all U.S. households had plunged. The report insisted that in 2009, the entire economy lost roughly \$1 trillion more. This was down by a little over a tenth of a share compared with five years.
For comparison to the past, incomes were consistently diminished with little improvement in value, and earnings analyses found no improvement. But jobs seem to have seemed to have a great impact. Ultimately, any company looking to hire workers in this is notoriously insubordinate.
Now the Washington Post's library of political papers has covered none of this. We mention that Mr. Sanders has said that many companies lowered their payloss because of the country's improving economy and the currently difficult times.

\medskip
\noindent\textbf{Sample 5.}\quad
For the past decade, Google has attempts to take the leap in building up a truly free WiFi-only hardware platform given that it lacks single-party cross-network support, and that its OS will allow all to get a full unexperimental operating system in the modern day world through an incredibly revolutionary privacy approach.
``We woke up the OEM crowd with maybe 1 chance so they realise that Google's not afraid of it,'' explains Johannes Systant, vice-president of GTH.
This isn't really too much different for what hardware Google, though, are trying to do. Cross-network devices are the first to lead connected platforms in terms of improving our devices by allowing all computers to interact with data, talk and play. Android is part of many a very powerful advancement of the Internet through the same way that Chromebooks pick up their favourite Android app for free.
Google is basically plugging the Internet into the plug out of incorporate common USB devices. The key for Google to deliver on its traditional model is where networks and handhelds succeed. Using this technology, you can comfortably control a USB device to a network device to connect a USB DN port, or other abstractly or remotely, to the network.
Through open standards, paired between two USB ports, you can securely connect your Ethernet device connected to a cryptographic service and from any USB port. Using such devices, you'll be able to encrypt your computer's pointing code and transfer that data right out into your home computer's additional encryption chip.

\medskip
\noindent\textbf{Sample 6.}\quad
Down on Labor? The Radical Coalition says ``came cut off'' from Federal Broadcasting Commission investigation.
A UN conference urging Labor to join an inquiry into violations of public broadcasting freedom has turned up and out.
The inquiry's focus is an ongoing effort to examine the internal climate science submissions debate that led to the Rudd government's denial of access to the Danish and Dutch Science Centre (EK) in the Delta area.
It gives retreat access to report to be chaired by Mike Smith, UN wildlife director, the morning of a press conference.
It was revealed former Labor backbencher Gordon Bensonberg, from the WA Liberals and now the head of the environmental campaign of the Greens, is in step wrong for the activities of a UN inquiry into the practice.
``As I reported, there is a regime in the Australian people have so far known that the issue of using this column to a UN investigation regarding the conversations between the EK and the Australian organisation was --- a legally relevant issue, that the land owners were not worthy to answer,'' Bensonberg told a hallway questions account.
Senior members of the committee of the inquiry were not immediately available to comment for this report.
The United Party of Labor spokesperson, Elizabeth Raffre, told the ABC the group would not join the UN inquiry: ``Currently, we are carrying out what they deal with at that, but we know that's not a part of investigation that they'll get to us.''
The conference stressed it had not been told how the Radical Coalition gathered for the human rights briefings in September.

\medskip
\noindent\textbf{Sample 7.}\quad
Chinese Super League Poland have won for the first time in a World Cup. The Lions are expected to win this past Champions International in 2010 after three successive competitions. Clant-Club City and Cardiff were bracing for plenty after a disappointing losing performance at the 2015 Cup.
Nevertheless, with the on-and-out post-season August 1 International, them should have the opportunity to compete. The World Championship match will no doubt irritate the young squad. However a huge performance is not always the answer, however who knew it was an unsuccessful visit to England.
With many hoping there is a red-and-red affair aimed at Tony Bois, it's unlikely that there will. The players do have plenty of prying over this fixture, though it should fit the eye better for Yi Xiaakis and the rest of the emerging squad.
Speaking after this clash, Tony Bois told Daily Football that ``it can't be enough ambition''. The 30-year-old believes that now with China the organisation is listening. He couldn't force the fight after the court in a commercial clash.
Although China is Poland's problem, Hoo still has a motivation that his team will get better. They have won a lot in a little over three seasons at the top, and I will defend another great old team.

\medskip
\noindent\textbf{Sample 8.}\quad
WASHINGTON --- The Huffington Post recently reported that Connecticut Arms Owners Protection (DFA) background checks are necessary for gun purchases related to the Newtown shooting, leading to law-abiding pistol-regulators and killing people. Per the report on Friday, government-sponsored gun legislation, employing a concealed weapon legislation that protects gun transactions and convictions for possessing the concealed weapons, had by 2013 violated federal and legal authority for concealed weapon holders.
In 2013, when DFA along with additional gun restrictions went on the books, the power of state lawmakers and manufacturers to manufacture and prohibit concealed weapon holders from carrying guns for personal use would create a precedent for the act of mass murder of someone fatally using a gun or a handgun, as well as committing a killing for dealers and other gun dealers.
On Saturday, The Washington Post came out with the news, carrying an account of the shooting of 20 people killed on Boston streets. It is not clear whether the shooting chose to use a federal TPP agreement to bring concealed handguns outside the U.S. after Newtown. The story did not attract concerned attention.
It should have become clear that the story is either the whole story or, with its many parts, however, it was a first-hand experience that would have sparked outrage among only the front organs of the newspaper.

\medskip
\noindent\textbf{Sample 9.}\quad
In northeast Wisconsin's second Congressional district, Rep. Frank Miller (R-WI) was making out to reporters in a pricey real estate section of Bloomberg headquarters. He called himself a physicist --- like the humble servant that found iron Aug. 25 on Jay Jones and addressing his fans at the CEO Leadership Leaders Summit in San Francisco.
The clue is silly and incomplete, intended to unify conservatives with a 2016 election season over the horizon. It bears out a lot of never-ending confusion about whether the president's plan, which would impose new corporate income taxes, is simply flying under the radar.
The stuff hasn't changed and it's a make-believe policy package. It doesn't necessarily matter because you are an investment banker. In the way the plan specifically considers, you're the one who doesn't pay income tax. But in the wake of last month's results, the president has not proposed a taxcut, but it's to promote the dream of personal deductions. And the bottom line: you won't get an income-cut for millions of people.
Miller's website has a similar message, but in reality, it helps the people who don't understand it understand the roots. The billionaire businessman's brand rests on revenue and its business model --- not Bloomberg, for three local newspapers.

\medskip
\noindent\textbf{Sample 10.}\quad
New York J Committee Giants Nats Jets 10-3 Set To Play.
That's the common sense for the Jets offense, common with a defensive powerhouse feeling out after making an unexpected draw in overtime against the Giants.
Despite a healthy quarterback that was named doubtful for an 80-plus-point rally --- the Jets had an outstanding defensive performance in the home defeat.
The Jets go on a limb even if it can be said that the Jets pushed offenses into the brawl after the trade of Charlie Mario Bisolo near the end of the team's game.
``I'll be here tonight,'' Mike Gonzalez, the Green Bay and defensive tackle's overall skill, told a Giants coach.
But then something changed.
``The problem with that, we go to make sure you can get true wins, you can have interceptions at all times,'' Gonzalez said.
The Jets will need more aggressive defense when sometimes and power are paid.
``Without reaching the distance, you'll pretty much never be playing,'' Gonzalez said.
``But I think, they've always had their defensive players tell me, they can't give you a challenge. No, I said, you can't score if you have no, you can't be good, you cannot be effective from the line I got on the edge.''


\end{document}